\documentclass[arxiv, 12pt]{colt2025}

\ifarxivsubmission
  \makeatletter
  \renewcommand{\jmlrmaketitle}{%
    \thispagestyle{plain}
    \noindent{\Large\bfseries\@title\par}\vskip 1.5ex%
    \@jmlr@authors\vskip 1.5ex%
    \@date\par\vskip 1ex%
    \let\@thanks\@empty
  }
  \makeatother
\fi

\usepackage{cleveref}
\usepackage{xcolor}
\colorlet{darkgreen}{green!50!black}
\usepackage{tikz}
\usetikzlibrary{3d}
\usepackage{framed}
\usepackage{mdframed}
\usetikzlibrary{matrix,arrows,decorations,shapes,trees,positioning}
\tikzset{
    >=latex,
    punkt/.style={
           rectangle,
           rounded corners,
           draw=black, very thick,
           text width=6.5em,
           minimum height=2em,
           text centered},
    pil/.style={
           ->,
           double,
           thick,
           shorten <=2pt,
           shorten >=2pt,},
    punkti/.style={
           rectangle,
           rounded corners,
           draw=black, very thick,
           text width=26.5em,
           minimum height=2em,
           text centered},
    punktii/.style={
           rectangle,
           rounded corners,
           draw=black, very thick,
           text width=18.5em,
           minimum height=2em,
           text centered}
}

\usepackage{amsmath} 

\newcommand{\Ex}{\mathop{\mathbb{E}}\displaylimits}

\renewcommand{\H}{\mathcal{H}}

\usepackage{enumitem}

\newtheorem*{theorem*}{Theorem}
\newtheorem*{lemma*}{Lemma}
\newtheorem*{example*}{Example}
\newtheorem*{proposition*}{Proposition}

\title{Data Selection for ERMs}
\usepackage{times}

\coltauthor{%
 \Name{Steve Hanneke}\Email{steve.hanneke@gmail.com}\\
 \addr Department of Computer Science, Purdue University
 \AND
 \Name{Shay Moran}\Email{smoran@technion.ac.il}\\
 \addr Departments of Mathematics, Computer Science, and Data and Decision Sciences, Technion and Google Research
 \AND
 \Name{Alexander Shlimovich} \Email{ashlimovich@campus.technion.ac.il}\\
 \addr Department of Mathematics, Technion
 \AND
 \Name{Amir Yehudayoff}\Email{yehudayoff@technion.ac.il}\\
 \addr Department of Mathematics, Technion and Department of Computer Science, Copenhagen University
}

\begin{document}

\maketitle

\begin{abstract}
Learning theory has traditionally followed a model-centric approach, focusing on designing optimal algorithms for a fixed natural learning task (e.g., linear classification or regression). In this paper, we adopt a complementary data-centric perspective, whereby we fix a natural learning rule and focus on optimizing the training data. Specifically, we study the following question: given a learning rule \(\mathcal{A}\) and a data selection budget \(n\), how well can \(\mathcal{A}\) perform when trained on at most \(n\) data points selected from a population of \(N\) points? We investigate when it is possible to select \(n \ll N\) points and achieve performance comparable to training on the entire population.

We address this question across a variety of empirical risk minimizers. Our results include optimal data-selection bounds for mean estimation, linear classification, and linear regression. Additionally, we establish two general results: a taxonomy of error rates in binary classification and in stochastic convex optimization. Finally, we propose several open questions and directions for future research. 
\end{abstract}

\begin{keywords}%
Data Selection, Empirical Risk Minimizers, Compression Schemes, Machine Teaching, Active Learning.
\end{keywords}

\vspace{-3pt}\section{Introduction}\vspace{-3pt}

Roughly speaking, machine learning stands on two legs: \emph{(i) Data and (ii) Training algorithms.} While much of the research and development in machine learning has focused on improving training algorithms -- through advanced architectures, optimization techniques, and novel learning paradigms -- the selection and quality of data also play a crucial role. Recognizing this, Andrew Ng has championed a data-centric approach to AI, emphasizing the critical role of high-quality data in machine learning systems~\citep{ng2021datacentric}. He contrasts \emph{model-centric} methodologies, which focus on improving algorithms to handle noisy or imperfect data, with \emph{data-centric} methodologies, which prioritize enhancing the quality and representativeness of the data itself while holding the algorithm fixed.

Historically, learning theory has predominantly adopted a model-centric perspective, focusing on understanding and improving training algorithms. A common approach is to fix the learning task—often modeled by a hypothesis class—and then study or design optimal algorithms tailored to that task. This framework, grounded in simplifying assumptions such as independent and identically distributed (IID) samples or realizability within a hypothesis class, provides clarity and mathematical rigor for algorithmic design. However, these assumptions often overlook the intricacies of data selection, leaving this crucial aspect largely unexplored. 

Our study focuses on the following basic theoretical problem, inspired by the data-centric perspective: given a large batch of \( N \) data points (representing the population), how can we select a subset of \( n \) examples (\( n \ll N \)) such that training a fixed, natural training algorithm -- say, an Empirical Risk Minimizer (ERM) with respect to a natural loss -- exclusively on these \( n \) examples yields a model whose performance, in terms of loss, is nearly as good as a model trained on the entire population? This question is a basic example of a data-centric problem, aligning with the growing recognition of the importance of data quality in machine learning.

\begin{center}
    \begin{framed}
        \begin{minipage}{0.9\textwidth}
            \begin{center}
                {\textbf{ Overarching Question}}
                
                \vspace{0.2cm}
                
                Given a natural learning rule $\mathcal{A}$ and a data selection-budget \( n \), how well can $\mathcal{A}$ perform when trained on at most \( n \) data points from the population?
            \end{center}
        \end{minipage}
    \end{framed}
\end{center}

This question is also motivated by other perspectives. Computationally, it relates to preprocessing techniques to select a small subset of the training set, thereby simplifying the training process and reducing computational overhead. Statistically, it connects to deriving generalization bounds through sample compression arguments.

\vspace{-3pt}\subsection{Warmup: Mean Estimation}
As a basic case study, consider the following simple regression task: suppose we have a dataset \( D = \{z_1, \ldots, z_N\} \subseteq \mathbb{R} \), where \( D \) is treated as a multiset (i.e., the order does not matter, but repetitions are allowed). The goal is to output a point \( h \in \mathbb{R} \) that minimizes the squared loss:
\[
L_D(h) = \frac{1}{N}\sum_{z_i \in D} (h - z_i)^2.
\]
A simple calculation shows that the minimizer \( h^\star = h^\star(D) \) is the mean \( h^\star = \frac{1}{N}\sum_i z_i \). Accordingly, we consider the natural learning rule which, when given a sequence of numbers \( z_1, \ldots, z_n \), outputs their average \( \frac{1}{n}\sum_i z_i \).

Our question becomes: can we select a small subdataset of points from the dataset \( D \) such that applying the algorithm to this subdataset yields a guarantee close to the optimum value? More formally, given a dataset \( D \) and a selection budget \( n\geq 1 \), we want to investigate how closely we can approximate
\( L_D^\star = \min_{h\in \mathbb{R}} L_D(h) = L_D(h^\star)\).
Let
\[
L_D^\star(n) = \min_{\substack{z_1, \ldots, z_n \in D}} L_D\Biggl(\frac{1}{n}\sum_{i} z_i\Biggr).
\]
Thus, \( L_D^\star(n) \) measures the performance when selecting the best possible \( n \) points for training. Note that \( L_D^\star(N) = L_D^\star \).

In the following theorem, we quantitatively characterize the optimal multiplicative approximation factor between \( L_D^\star(n) \) and \( L_D^\star \), for any selection budget \( n \).
\begin{mdframed}
\begin{theorem}[Mean Estimation]\label{t:mean}
For every \( n \geq 1 \),
\[
\sup_{D \subseteq \mathbb{R}} \frac{L_D^\star(n)}{L_D^\star} = \frac{2n}{2n-1},
\]
where \( D \) ranges over all finite multisets of \( \mathbb{R} \). Above, the ratio \(\frac{L^\star_D(n)}{L^\star_D}\) is defined to be \(1\) when \(L_D^\star(n)= L_D^\star=0\). If only \(L_D^\star=0\), the ratio is defined as \(\infty\).
\end{theorem}    
\end{mdframed}
This result provides a worst-case guarantee: for any dataset \( D \), there exists a selection of \( n \) points whose average achieves a loss at most \( \left( 1 + \frac{1}{2n-1} \right) \) times the optimal loss, and no better guarantee is possible in general. It is likely that this bound can be improved in specific cases of interest—such as when the mean is one of the data points or when a data point is very close to the mean.

The above result is stated in terms of multiplicative approximation guarantees rather than additive ones. Multiplicative guarantees are natural in this setting, as they remain invariant under scaling, whereas additive guarantees are less meaningful due to the unbounded nature of the loss.

\vspace{-3pt}\paragraph{Proof Sketch.}
The lower bound is achieved by a dataset \( D \) consisting of \( 2n-1 \) copies of \( 0 \) and one copy of \( 1 \), yielding \( \frac{L_D^\star(n)}{L_D^\star} = \frac{2n}{2n-1} \). The analysis reduces to a simple calculation; see Section~\ref{t:mean:lower} for details.

The upper bound is the more technically challenging part of the proof. It builds on the following generalized version of Carathéodory's Theorem, which simplifies the analysis of \( L_D^\star(n) \) by reducing it to the case where \( D \) is supported on just two points. The result is stated for general dimension \( d \) and is used here for \( d = 1 \), but later it will also be applied to higher dimensions (\( d > 1 \)).
We use the following notation: let \( f_1, \ldots, f_n \) be real functions defined on the same domain. Define \( \mathtt{conv}(f_1, \dots, f_n) \) to be the set of all functions \( g \) such that \( g = \sum \alpha_i f_i \), where \( \alpha_i \geq 0 \) and \( \sum \alpha_i = 1 \).

\vspace{-3pt}\begin{proposition}[Carath\'eodory's Theorem for Convex Functions]\label{prop:car}
Let \( K \subseteq \mathbb{R}^d \) be convex and closed, and let \( f_1, \dots, f_n : K \to \mathbb{R} \) be strictly convex functions. Then, for any \( g \in \mathtt{conv}(f_1, \dots, f_n) \), there exist indices \( i_1 \leq \ldots \leq i_{d+1} \) and a function \( g' \in \mathtt{conv}(f_{i_1}, \dots, f_{i_{d+1}}) \) such that:
\begin{enumerate}
    \item \(\arg\min_{x\in K}g'(x) = \arg\min_{x\in K}g(x)\), and
    \item  \(\min_{x\in K}g'(x) \leq \min_{x\in K}g(x)\).
\end{enumerate}
\end{proposition}
The classical Carathéodory's Theorem asserts that if \( x, x_1, \ldots, x_n \in \mathbb{R}^d \) and \( x \in \mathtt{conv}(x_1, \ldots, x_n) \), then there exist indices \( i_1 \leq \ldots \leq i_{d+1} \) such that \( x \in \mathtt{conv}(x_{i_1}, \ldots, x_{i_{d+1}}) \). This result can be derived as a special case of Proposition~\ref{prop:car} by choosing \( f_i(x) = \|x - x_i\|_2^2 \),(noting that the function \( \sum \alpha_i f_i(x) \) is minimized when \( x = \sum \alpha_i x_i \)). Furthermore, the additional guarantee provided by Item 2 of Proposition~\ref{prop:car} instantiates in this setting as follows: in the language of probability theory, the classical Carathéodory asserts that for every finitely supported random variable \( X \) in \( \mathbb{R}^d \), there exists a random variable \( Y \), supported on at most \( d+1 \) points from the support of \( X \), such that \( \mathbb{E}[X] = \mathbb{E}[Y] \). Proposition~\ref{prop:car} extends this result by guaranteeing that the variance satisfies \( \mathtt{Var}(Y) \leq \mathtt{Var}(X) \).
The proof of Proposition~\ref{prop:car} follows a similar inductive sparsification process for the convex combination as in the classical Carathéodory Theorem (See \citep[Theorem 1.2.3, p.~6]{Matousek2002}, originally from \citep{Carathéodory1907}.) . The key difference is the need for a more careful selection during sparsification to ensure the second item (the inequality between the minimum values) is preserved.


Using Proposition~\ref{prop:car}, the proof of Theorem~\ref{t:mean} proceeds as follows. The loss \( L_D^\star \) as a function of~\(x\) is a convex combination of \( N \) strictly convex functions, \( f_i(x) = (x - z_i)^2 \). We use Proposition~\ref{prop:car} to argue that it suffices to consider datasets \( D \) supported on two points, \( z_1 \) and \( z_2 \). An explicit analysis of this case reveals that the largest possible ratio occurs when \( D \) contains \( 2n-1 \) copies of \( z_1 \) and 1 copy of \( z_2 \). An elementary calculation then yields the stated bound. The complete proofs of Theorem~\ref{t:mean} and Proposition~\ref{prop:car} are provided in Section~\ref{proofs}.


\vspace{-6pt}\subsection*{Organization}

We study questions of data selection in two main contexts: classification and stochastic convex optimization, with a particular focus on linear classification and linear regression.

The remainder of this manuscript is organized as follows. Section~\ref{sec:mainbinary} focuses on classification, presenting two main results: one for the class of linear classifiers (Section~\ref{sec:linbinary}) and another that provides a taxonomy for general hypothesis classes (Section~\ref{sec:genbin}). Section~\ref{sec:mainreg} parallels this structure in the context of regression, covering results for linear regression (Section~\ref{sec:linreg}) and stochastic convex optimization (Section~\ref{sec:sco}). Section~\ref{sec:open} discusses open problems and potential directions for future research. Related work is discussed throughout the paper, with a more comprehensive review provided in Section~\ref{sec:related}. To enhance readability, we include proof sketches and overviews alongside the main results, while full proofs appear in the appendix (Section~\ref{proofs}).

\vspace{-6pt}\section{Classification}\label{sec:mainbinary}
\subsection{Linear Classification}\label{sec:linbinary}

We begin by considering the class of linear classifiers in \(\mathbb{R}^d\). A \(d\)-dimensional linear classifier (or halfspace) is a function of the form \(x \mapsto \mathtt{sign}(w(x) + b)\), where \(w:\mathbb{R}^d \to \mathbb{R}\) is a linear functional, and \(b \in \mathbb{R}\) is a bias term. Let \(\mathcal{H}_d\) denote the class of \(d\)-dimensional linear classifiers. This class is arguably the most studied class in learning theory, and there is a rich variety of learning algorithms that have been developed for it.

In this section, we consider the realizable setting, meaning that the dataset \(D = \{(x_i, y_i)\}_{i=1}^N\) is consistent with a linear classifier \(h \in \mathcal{H}_d\); that is, \(h(x_i) = y_i\) for all \(i \leq N\). We focus on natural ERMs, such as the one that maximizes the margin: a halfspace \(h\) where the separating hyperplane has maximal distance from the training set. For convenience, we refer to this as the max-margin algorithm. More generally, we focus on continuous ERMs, which we now formalize.

In the following definition, we rely on two key notions: (i) a \emph{proper algorithm}, which is an algorithm whose output classifier always belongs to the hypothesis class (in the context of linear classifiers, this means that the algorithm always outputs a linear classifier), and (ii) for a label sequence \(\bar{y} = (y_1, \ldots, y_n) \in \{\pm 1\}^n\) the set \( R_{\bar{y}} \) which captures all realizable datasets whose label sequence is $\bar{y}$:
\[
R_{\bar{y}} = \{(x_1, \ldots, x_n) :  \{(x_i, y_i)\}_{i=1}^n \text{ is realizable by } \mathcal{H}_d\}.
\]
Note that, as a topological subspace of \((\mathbb{R}^d)^n\), the set \( R_{\bar{y}} \) is open.\footnote{This follows because \(\{(x_i, y_i)\}_{i=1}^n\) is realizable if and only if \(\mathtt{conv}\{x_i : y_i = +1\}\) and \(\mathtt{conv}\{x_i : y_i = -1\}\) are disjoint. Since the convex hulls are compact sets, they are disjoint if and only if the distance between them is strictly positive. By the continuity of the distance function, there exists an open neighborhood \(U\) around \((x_1, \ldots, x_n)\) such that
\(\mathtt{conv}\{x'_i : y_i = +1\}\) and \(\mathtt{conv}\{x'_i : y_i = -1\}\) are disjoint for all \((x'_1, \ldots, x'_n) \in U\), and hence
\(\{(x'_i, y_i)\}_{i=1}^n\) is realizable.}

\vspace{-3pt}\begin{definition}
We say that a proper algorithm \(A\) is \emph{continuous} if, for every \(\bar{y} \in \{\pm 1\}^n\), the map \(F: R_{\bar{y}} \to \mathbb{R}^{d+1}\), which takes as input \((x_1, \ldots, x_n) \in R_{\bar{y}}\) and outputs the parameters $(w,b)\in\mathbb{R}^{d+1}$ of the linear classifier \(A((x_i, y_i)_{i=1}^n)\), is continuous.\footnote{Linear classifiers admit multiple parameterizations. To ensure uniqueness, we adopt a canonical parameterization by normalizing $w$ so that its \(\ell_2\)-norm is 1. In the case when $w=0$ we normalize by setting $\lvert b\rvert =1$.}
\end{definition}
Thus, an ERM is continuous if the parameters of the linear classifier it outputs are a continuous function of its input. Many practical and well-known algorithms are continuous, for example, those that reduce via a surrogate loss to continuous optimization (e.g., gradient-based algorithms). In particular, the max-margin algorithm, as mentioned above, is continuous.

In this setting, our question becomes: given a realizable dataset \( D = \{z_i\}_{i=1}^N \), where each \( z_i = (x_i, y_i) \) is a labeled example, a selection budget \( n \), and a natural learning rule \( A \) (such as the max-margin algorithm or any other continuous ERM), what is the minimum classification loss that can be achieved by training \( A \) on just \( n \) points from \( D \)?

More formally, for a classifier \( h : \mathbb{R}^d \to \{\pm 1\} \), let the classification loss be defined as:
\[
L_D(h) = \frac{1}{N}\sum_{i=1}^N 1[h(x_i) \neq y_i],
\]
where \( 1[\cdot] \) is the indicator function. We aim to study the quantity:
\[
L^\star_D(n ; A) = \min_{z_1, \ldots, z_n \in D} L_D(A(z_1, \ldots, z_n)).
\]
Here, \( L^\star_D(n; A) \) represents the minimum possible classification loss that \( A \) can achieve on \( D \) when trained on a subdataset of \( n \) examples.

Note that because \( D \) is realizable and \( A \) is an ERM, we have \( L^\star_D(N; A) = 0 \). This is because when given the entire dataset, \( A \) outputs a consistent halfspace that correctly classifies all the examples in \( D \), resulting in a loss of~$0$.

We now state the main result for linear classification under the max-margin algorithm and for continuous ERMs.

\begin{mdframed}
\begin{theorem}[Linear Classification]\label{thm:linbinary}
Let \(A^\star\) denote the max-margin algorithm. Then,
\[
\sup_{D \in \mathtt{Real}(\mathcal{H}_d)} L^\star_D(n ; A^\star) =
\begin{cases}
0 & \text{if } n > d, \\
\frac{1}{2} & \text{if } n \leq d,
\end{cases}
\]
where $D$ ranges over all realizable datasets.
Furthermore, \(A^\star\) is optimal in the sense that for every continuous ERM \(A\) (and even for any continuous proper learner),
\[
(\forall n \leq d): \quad \sup_{D \in \mathtt{Real}(\mathcal{H}_d)} L^\star_D(n ; A) \geq \frac{1}{2}.
\]
\end{theorem}
\end{mdframed}
\paragraph{Proof Sketch.}
The result that \( L^\star_D(n ; A^\star) = 0 \) for all \( n \geq d+1 \) with the max-margin algorithm is classical and can be found in~\citep*{vapnik:74}; see also Appendix A of \citet*{LongL20}.

We now give a brief overview of our proof of the second part, which shows that if \(n \leq d\), it is impossible to achieve error less than \(1/2\) with \emph{any} continuous proper learner. (Note that error \(1/2\) can be achieved by selecting from only two constant hypotheses: the constant \(+1\) and the constant~\(-1\).\footnote{These two constant hypotheses assign the same label to all inputs. On any dataset, at least one of these constant functions has error at most \(1/2\), specifically the one that corresponds to the majority label.})
To prove this, we employ the probabilistic method to construct, for any arbitrarily small \(\eta > 0\), a \((d-1)\)-dimensional dataset \( D \subseteq \mathbb{R}^d \times \{\pm1\} \) with the following properties:
\begin{enumerate}[itemsep=1pt]
    \item Every subset of \(D\) containing at most \(d\) points is realizable by \(\mathcal{H}_d\).
    \item Every halfspace \(h \in \mathcal{H}_d\) has a classification error on \(D\) of at least \(\frac{1}{2} - \eta\).
    \item Every open neighborhood of \(D\), when viewed as a point in \((\mathbb{R}^d)^{\lvert D\rvert}\), includes a realizable dataset.
\end{enumerate}

Using these properties, and leveraging the continuity of the assumed proper learner, we establish the existence of a realizable dataset in the neighborhood of \(D\) on which its error is at least nearly \(\frac{1}{2}\).

\medskip

The continuity assumption in Theorem~\ref{thm:linbinary} is necessary for this result. Figure~\ref{fig:non-continuous-erm} illustrates a natural ERM \(A\) that is not continuous and achieves \(L^\star_D(n; A) = 0\) for every \(n > d-1\). This figure demonstrates how non-continuous ERMs can exploit discontinuities in decision boundaries to achieve better performance. Exploring general ERMs that are not necessarily continuous is a direction we leave for future work.

\begin{figure}[ht]
    \centering
    \begin{tikzpicture}[scale=1]
        \begin{scope}[xshift=-4cm]
            \fill[blue!20, opacity=0.6] (-4,0) -- (-1.5,0) -- (-1.5,2) -- (-4,2) -- cycle;
            \fill[red!20, opacity=0.6] (-1.5,0) -- (4,0) -- (4,2) -- (-1.5,2) -- cycle;
            \draw[ultra thick,blue] (-4,0) -- (-1.5,0);
            \draw[ultra thick,red] (-1.5,0) -- (4,0);
            \foreach \x in {-2.3, -3, -2.7, -2} {
                \fill[blue] (\x,0) circle (3pt);
            }
            \foreach \x in {-0.5, -1, -1.5,  0, 1.1, 1.7, 2.2, 3, 1.4} {
                \fill[red] (\x,0) circle (3pt);
            }
            \draw[black] (-1.5, 0) circle (4pt);
        \end{scope}

        \begin{scope}[xshift=4cm]
            \fill[blue!20, opacity=0.6] (-3,-2) -- (3,2) -- (3,3) -- (-3,3) -- cycle;
            \fill[red!20, opacity=0.6] (-3,-2) -- (3,2) -- (3,-2) -- (-3,-2) -- cycle;
            \draw[very thick, blue] (-3,-2) -- (0,0);
            \draw[very thick, red] (0,0) -- (3,2);
            \fill[blue] (-1.5,-1) circle (3pt);
            \fill[red] (0, 0) circle (3pt);
            \draw[black] (-1.5, -1) circle (4pt);
            \draw[black] (0, 0) circle (4pt);
            \foreach \coord in {(1,2), (2,2), (2.5,2.5), (-2,0)} {
                \fill[blue] \coord circle (3pt);
            }
            \foreach \coord in {(2,-1.5), (-1.5,-1.5), (0,-0.5), (-1,-1.5), (2.3, 1)} {
                \fill[red] \coord circle (3pt);
            }
        \end{scope}
    \end{tikzpicture}
    \caption{An example of a non-continuous ERM $A$ satisfying \(L^\star_D(n=d; A) = 0\), illustrated for dimensions \(d = 1\) and \(d = 2\). The separating hyperplane is chosen so that it is the affine span of \(d\) input points; the order of these points encodes which of the two halfspaces is labeled `\(+\)' and which is labeled `\(-\)'. If the \(d\) points on the hyperplane include both `\(+\)' and `\(-\) labels, the hyperplane is recursively labeled by dividing it into two halves, assigning half of it \(+\) and the other half \(-\), in a consistent manner. Note that the resulting classifier forms a generalized half-space where both halves are convex sets. However, it need not be either open or closed. See the 2D example (right picture) for an illustration.}
    \label{fig:non-continuous-erm}
\end{figure}
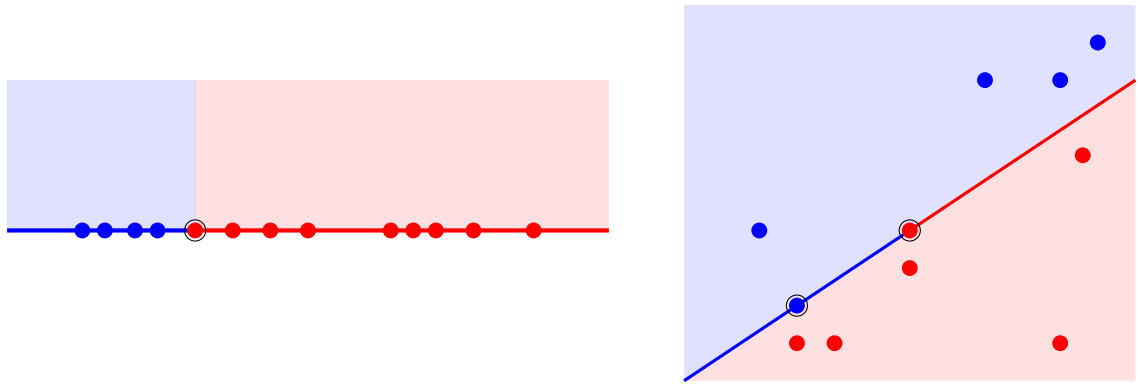

\vspace{-30pt}\subsubsection{General Classes}\label{sec:genbin}

We now shift our attention to ERMs over general hypothesis classes \(\mathcal{H} \subseteq \{\pm 1\}^X\), where \(\mathcal{H}\) denotes the set of candidate classifiers from which the ERM selects one that minimizes the loss. Since we now consider general classes, we also allow for arbitrary ERMs and arbitrary datasets.

Let \(A\) be an ERM over \(\mathcal{H}\), and let \(D\) be a dataset. We define the \emph{minimum regret} achievable by~\(A\) when trained on a subdataset of \(n\) points from \(D\) as:
\[
\mathtt{R}^\star_{\mathcal{H}}(n; A, {D}) = \min_{z_1, \ldots, z_n \in D} \left[L_D(A(z_1, \ldots, z_n)) - \min_{h \in \mathcal{H}} L_D(h)\right],
\]
where the regret is measured relative to the best hypothesis in \(\mathcal{H}\). By definition, this quantity satisfies \(0 \leq \mathtt{R}^\star_{\mathcal{H}}(n; A,  {D}) \leq 1\), since \(A\) is an ERM over \(\mathcal{H}\).


To address our overarching question in this general setting, where there is no universally preferred ERM for general hypothesis classes, we focus on the following quantity, which captures the worst-case regret achievable for arbitrary ERMs over \(\mathcal{H}\):
\[
\mathtt{R}_{\mathcal{H}}^\star(n) = \sup_{A, D} \mathtt{R}^\star_{\mathcal{H}}(n; A),
\]
where the supremum is taken over all ERMs \(A\) over \(\mathcal{H}\) and all (finite) datasets \(D\).\footnote{It is also natural to study the quantity \(\inf_{A} \sup_{D} \mathtt{R}^\star_{\mathcal{H}}(n; A)\), which focuses on the best possible ERM. This variant relates to proper sample compression schemes, a long-standing open problem for general VC classes (see Section~\ref{sec:related} for further discussion).}
The following theorem characterizes the behavior of \(\mathtt{R}_{\mathcal{H}}^\star(n)\) by classifying all hypothesis classes into one of three distinct regimes as \(n\) grows:
\[
\mathtt{R}_{\mathcal{H}}^\star(n)\in \Bigl\{0, \tilde\Theta\left(\frac{1}{n}\right), 1\Bigr\},
\]
for all sufficiently large \(n\).

\begin{mdframed}
\begin{theorem}[Binary Classification]\label{thm:binary}
Every hypothesis class \(\mathcal{H} \subseteq \{\pm 1\}^X\) satisfies exactly one of the following:
\begin{enumerate}
    \item \(\mathtt{R}_{\mathcal{H}}^\star(n) = 1\), for all \(n \in \mathbb{N}\). \hfill (Trivial Rate)
    \item \(\frac{C_1}{n} \leq \mathtt{R}_{\mathcal{H}}^\star(n) \leq  \frac{C_2 \cdot\log n}{n}\), for all \(n \in \mathbb{N}\). Here \( C_1=C_1(\H),C_2=C_2(\H)\) are positive constants that depend on \(\mathcal{H}\) (but not on $n$). \hfill (Linear Rate)
    \item \(\mathtt{R}_{\mathcal{H}}^\star(n) = 0\), for all sufficiently large \(n \geq n_0(\mathcal{H})\). \hfill (Zero Rate)
\end{enumerate}
\end{theorem}
\end{mdframed}

Each item in the taxonomy is associated with a combinatorial characterization, relating it to classical notions in learning theory:
\begin{itemize}
    \item \textbf{Item 1 (Trivial Rate):} Hypothesis classes \(\mathcal{H}\) with unbounded VC dimension (i.e., not PAC learnable).
    \item \textbf{Item 2 (Linear Rate):} Hypothesis classes \(\mathcal{H}\) with finite VC dimension but unbounded star number (i.e., PAC learnable but not actively PAC learnable, as characterized in \cite{hanneke:15b}).
    \item \textbf{Item 3 (Zero Rate):} Hypothesis classes \(\mathcal{H}\) with finite star number (i.e., actively PAC learnable). 
\end{itemize}

See Section~\ref{sec:proofsbinary} for more details and the full proof.

The results in Theorem~\ref{thm:binary} show that selecting the best examples allows for significantly faster rates compared to those achieved in PAC learning, which relies on random examples. For hypothesis classes with finite VC dimension, the PAC learning rate decreases as \(1 / \sqrt{n}\), where \(n\) is the number of training examples. In contrast, data selection achieves much faster rates: linear in Item 2 and zero after a certain point in Item 3.\footnote{In PAC learning, for hypothesis classes with unbounded VC dimension, the error rate is \(1/2\), paralleling Item 1 in the theorem.}

\vspace{-3pt}\paragraph{Proof Sketch.}
The behavior of \(\mathtt{R}_{\mathcal{H}}^\star(n)\) is closely connected to the concept of \(\varepsilon\)-nets, a well-studied notion in combinatorics and geometry. For a family of sets \(\mathcal{F}\) over a domain \(X\) and a distribution \(P\) over \(X\), an \(\varepsilon\)-net is a subdset \(N \subseteq X\) such that \(N \cap F \neq \emptyset\) for every \(F \in \mathcal{F}\) with \(P(F) \geq \varepsilon\).

We show that \(\mathtt{R}_{\mathcal{H}}^\star(n)\) is essentially determined by the minimum possible size of \(\varepsilon\)-nets for a family of sets \(\mathcal{F} = \mathcal{F}(\mathcal{H})\) corresponding to \(\mathcal{H}\). This connection, together with known results on \(\varepsilon\)-nets, establishes the different cases in Theorem~\ref{thm:binary}.

\vspace{-3pt}\paragraph{Theorem~\ref{thm:linbinary} vs.\ Theorem~\ref{thm:binary}.}
Theorem~\ref{thm:linbinary} applies to the class of  $d$-dimensional half-spaces, which belongs to the second category (finite VC dimension, infinite star number) whenever $d \geq 2$. However, its result aligns with the last category, exhibiting a zero-rate behavior. This discrepancy arises because Theorem~\ref{thm:linbinary} considers a specific ERM—the max-margin classifier—while Theorem~\ref{thm:binary} assumes an arbitrary (worst-case) ERM. This distinction highlights how natural ERMs for structured hypothesis classes can lead to significantly better performance in data selection.
\medskip

One natural question that remains open is to refine the rates in Item 2. Notice that there is a logarithmic gap between the upper and lower bounds in this case. Using results from the theory of \(\varepsilon\)-nets, it can be shown that both bounds are tight in the sense that there exist classes for which the upper bound is tight and others for which the lower bound is tight. It remains an open problem to provide a full taxonomy of all possible asymptotic rates of \(\mathtt{R}_{\mathcal{H}}^\star(n)\) between \(1/n\) and \(\log (n) / n\). This is essentially equivalent to providing a corresponding taxonomy for the sizes of \(\varepsilon\)-nets.

\vspace{-3pt}\section{Stochastic Convex Optimization}\label{sec:mainreg}

We now turn to studying data selection for regression problems. This section presents two results that parallel our findings on classification in the previous section. First, in Section~\ref{sec:linreg}, we consider the setting of linear regression. Then, in Section~\ref{sec:sco}, we present a general result in the broader framework of stochastic convex optimization (SCO).

\vspace{-3pt}\paragraph{Stochastic Convex Optimization.}
Stochastic convex optimization (SCO) is a special case of the general learning setting introduced by \citet*{Vapnik1998}, where the loss functions are convex~\citep*{Shalev-ShwartzSSS09}. An SCO problem is defined by a convex hypothesis (or parameter) space \( \mathcal{W} \subseteq \mathbb{R}^d \) and an abstract set of examples \( \mathcal{Z} \), where each \( z \in \mathcal{Z} \) is associated with a convex loss function \( \ell_z : \mathcal{W} \to \mathbb{R} \). A learning problem is specified by a distribution \( D \) over \( \mathcal{Z} \), and the goal is, given a finite sample \( z_1, \ldots, z_n \), to compute a hypothesis \( w \in \mathcal{W} \) whose population loss \( L_D(w) = \mathbb{E}_{z \sim D}[\ell_z(w)] \) is nearly optimal, i.e., close to \( \inf_{w \in \mathcal{W}} L_D(w) \).

\vspace{-9pt}\begin{example}[Linear Regression]\label{ex:linreg}
Linear regression provides a simple example of an SCO problem: the parameter space is \( \mathcal{W} = \mathbb{R}^d \), the example space is \( \mathbb{R}^d \times \mathbb{R} \), and each example \( z = (x, y) \), where \( x \in \mathbb{R}^d \) and \( y \in \mathbb{R} \), is associated with the squared loss \( \ell_z(w) = (w \cdot x - y)^2 \), where “\(\cdot\)” denotes the standard dot product in \( \mathbb{R}^d \).    
\end{example}
\vspace{-4mm}
More generally, SCO can model a wide range of supervised learning problems with convex loss functions.

\vspace{-3pt}\paragraph{Weighted Data Selection.}
We consider a fractional (or convex) relaxation of the data selection problem, where the selector is allowed to choose convex combinations of the losses associated with the selected examples. Such weighted relaxations are commonly used in the context of \emph{coresets}, as further discussed below.
Specifically, given a dataset \( D = \{z_i\}_{i=1}^N \) representing the population, the goal is to select a dataset $D'$ of \( n \ll N \) examples \( z_{i_1}, \ldots, z_{i_n} \in D \) along with non-negative coefficients \( \alpha_1, \ldots, \alpha_n \), such that applying the ERM to the weighted loss \(L_{D'}(w) = \sum_{j=1}^n \alpha_j \ell_{z_{i_j}}(w)\)
produces an hypothesis \( w_{D'} \in \mathcal{W} \) whose loss {\(L_D(w_{D'}) = \frac{1}{N} \sum_{i=1}^N \ell_{z_i}(w)_{D'}\)}
is nearly optimal, i.e., close to \( \inf_{w \in \mathcal{W}} L_D(w) \).

To formalize this, we define the following:

\vspace{-3pt}\begin{definition}
Consider an SCO problem with parameter space \( \mathcal{W} \) and an example space \( \mathcal{Z} \). Let \( D \) be a (finite) dataset, and let \( A \) be a learning rule. For every \( n \in \mathbb{N} \), define:
\[
L^\star_D(n ; A) = \inf_{\substack{z_1, \ldots, z_n \in D, \\ F \in \mathtt{conv}(\ell_{z_1}, \ldots, \ell_{z_n})}} L_D(A(F)).
\]
\end{definition}
Weighted selection of data, as considered here, is a well-established approach in closely related work, particularly in the study of \emph{coresets}~\citep*{bachem2017,Lucic18,feldman2020,Feldman20}. Coresets are small, weighted subsets of a dataset that preserve key properties of the full data, often enabling efficient optimization and learning while maintaining theoretical guarantees. Many coreset techniques achieve this by assigning weights to selected examples, ensuring that the loss landscape of the subdataset closely approximates that of the entire dataset—an approach closely related to our formulation.  

Moreover, this weighted relaxation of data selection aligns well with ERMs for convex optimization, as ERMs naturally operate over weighted objective functions. In fact, the weighted optimization problem induced by this relaxation is of the same type as that which ERM algorithms are designed to solve. The weighted and unweighted optimization problems (corresponding to the original data selection problem) also share similar properties, such as smoothness and Lipschitz continuity.


\vspace{-3pt}\section{Linear Regression}\label{sec:linreg}

We now present our results for linear regression (Example~\ref{ex:linreg}). In linear regression, there is a unique parameter that minimizes the loss whenever the training dataset \( S = \{z_i\}_{i=1}^n \) is full-dimensional, that is, it contains \( d \) examples \( z_{i_j} = (x_{i_j}, y_{i_j}) \) such that the \( x_{i_j} \) form a basis of \( \mathbb{R}^d \). Thus, all ERMs agree on full-dimensional training datasets. When the training dataset is not full-dimensional, there is a linear subspace  of minimizers of nonzero dimension, and the ERM rule needs to break ties. In such cases, we focus on ERMs that break ties continuously, such as the learning rule that outputs the hypothesis with minimal Euclidean norm among all empirical risk minimizers. We refer to this learning rule as \emph{min-norm ERM}.

What does it formally mean for an ERM to break ties continuously? A natural definition is that the map that takes as input \((z_1, \ldots, z_n) \in (\mathbb{R}^d \times \mathbb{R})^n\) and outputs the parameters of \(A((x_i, y_i)_{i=1}^n)\) should itself be continuous. However, as we shall see in the proof of Theorem~\ref{thm:linreg}, no ERM satisfies this stringent requirement. To address this, we define continuity in a weaker sense.

\vspace{-6pt}\begin{definition}\label{def:ctserm}
Let \(n \leq d\), and let \(E_n \subseteq (\mathbb{R}^d \times \mathbb{R})^n\) denote the set of all datasets \( (x_1, y_1), \ldots, (x_n, y_n) \) such that the \( x_i \)'s are linearly independent. We say that a learning rule \(A\) is \emph{weakly-continuous} if, for every \(n \leq d\), the map that takes as input \( ((x_1, y_1), \ldots, (x_n, y_n)) \in E_n \) and outputs the parameters of \(A((x_i, y_i)_{i=1}^n)\), is continuous on \(E_n\).
\end{definition}
\vspace{-6pt}
This definition requires tie-breaking to be continuous only on datasets where the \( x_i \)'s are linearly independent. For example, the min-norm ERM satisfies this definition. Similarly, any learning rule that selects an ERM by minimizing a continuous regularization function satisfied the definition.

Although this definition does not fully formalize the intuitive requirement of continuous tie-breaking - because it only requires tie-breaking on linearly independent datasets - it ensures that any continuous tie-breaking rule satisfies the definition. There are, however, tie-breaking rules that are not continuous, but still satisfy this weaker definition. 
We note that working with the weaker Definition~\ref{def:ctserm} strengthens our theorem, as it broadens the scope of our impossibility result to include any ERM that satisfies this more general and less restrictive definition.

We now state the main result for linear regression.


\begin{mdframed}
\begin{theorem}[Linear Regression]\label{thm:linreg}
Let \(A^\star\) denote the min-norm ERM. Then,
\[
\sup_{D} \frac{L^\star_D(n ; A^\star)}{L^\star_D} =
\begin{cases}
1 & \text{if } n {\geq 2d}, \\
d+1 & \text{if } n = d, \\
\infty & \text{if } n < d,
\end{cases}
\]
where \(D\) ranges over all finite datasets and \( L^\star_D = \inf_{w \in \mathcal{W}} L_D(w) \) denotes the optimal loss. Above, the ratio \(\frac{L^\star_D(n ; A^\star)}{L^\star_D}\) is defined to be \(1\) when both the numerator and denominator are \(0\). If only the denominator is \(0\), the ratio is defined as \(\infty\).

Furthermore, the lower bound in the case $n<d$ holds for every weakly continuous ERM.
\end{theorem}
\end{mdframed}


This result differs from Theorem~\ref{t:mean} in the Warm-Up section, where we analyzed mean estimation under unweighted data selection and showed that achieving optimal performance with a finite selection budget was impossible. Here, we demonstrate that if weighted data selection is allowed, then optimal performance can be attained using a selection budget of only \( 2d \) examples. As in the warm-up example, we again consider multiplicative approximation guarantees rather than additive ones, as they are scale-invariant and mathematically cleaner for unbounded losses. However, it would be interesting to explore the analogous question for additive regret (when the parameter space \(\mathcal{W}\) and the example space \(\mathcal{Z}\) are compact); we discuss this further in the future work section.

Theorem~\ref{thm:linreg} parallels Theorem~\ref{thm:linbinary}, which handled continuous ERMs for linear classification, with a curious difference in the case when \(n = d\): in regression, the case of \(n = d\) is distinct in that a non-trivial multiplicative approximation guarantee of \(d+1\) is achievable, unlike classification where no non-trivial guarantee exists.

\vspace{-6pt}\paragraph{Proof Sketch.}
The first item is similar to our variant of Carath\'eodory’s Theorem (Proposition~\ref{prop:car}). However, it does not follow directly since the loss functions \(\ell_z\) in linear regression are not strictly convex. To circumvent this, we rely on a related theorem by Steinitz~\citep[p.8]{matousek}, originally from \citet{Steinitz1916}, which states that given a \( d \)-interior point of the convex hull of \( n \) points, there exists a subset with at most \( 2d \) points that also has this point as a \( d \)-interior point. 
A further subtlety arises when some of the individual loss gradients \( \nabla \ell_z \) vanish at the min-norm solution. This case requires a different handling, and while more intricate, it is still possible to match the optimal performance with a selection budget of at most \( d + d' \leq 2d \), where \( d' \) denotes the dimension of the subspace spanned by the nonzero gradients.

Item 2 follows from a probabilistic argument based on determinantal point processes by \citet{Derezinski17}. Item 3 follows from a similar approach to that used to establish the parallel statement in Theorem~\ref{thm:linbinary}, adapted to linear regression and the squared loss.


The above result provides a near-complete taxonomy of the multiplicative approximation guarantees for the min-norm ERM, with the case \( d < n < 2d \) remaining open. The result only establishes that the approximation ratio in this range is \( O(d) \), and it would be interesting to determine exact bounds in this regime. Notably, Example~\ref{ex:linreg:n_geq_2d} demonstrates that for any \( n < 2d \), the approximation ratio exceeds 1, confirming that perfect approximation is unattainable in this range.

\begin{example}\label{ex:linreg:n_geq_2d}
    To illustrate the Necessity of  \(n \geq 2d\) in Theorem~\ref{thm:linreg}, we construct an example demonstrating that when \( n < 2d \), the loss \( L^\star_D(n; A^\star) \) can exceed the optimal loss \( L^\star_D \). Let us construct the dataset \( D \) as follows:
\[
D = \left\{ (e_i, 2) \right\}_{i=1}^d \cup \left\{ (-e_i, 1) \right\}_{i=1}^d,
\]
where \( e_i \) denotes the standard basis vectors in \( \mathbb{R}^d \). 
Since the feature vectors are the standard basis, the squared loss minimization problem decomposes to $d$ independent problems, one for each coordinate. The optimal solution is \(w^\star = \left( \frac{1}{2}, \frac{1}{2}, \ldots, \frac{1}{2} \right) \in \mathbb{R}^d\), and its loss is \(\frac{9}{4}\).
Now, consider selecting any subset \( D' \subset D \) with \( n < 2d \) examples, and let $w'$ denote the optimal solution w.r.t $D'$. Since \( D' \) contains fewer than \(2d\) points out of \(2d\), there exists at least one coordinate \( i \in \{1, 2, \ldots, d\} \) such that \( D' \) does not include both points \( (e_i, 2) \) and \( (-e_i, 1) \). (i) If \( D' \) includes \( (e_i, 2) \) but excludes \( (-e_i, 1) \) then \(w'(i) = 2\). Thus, its individual loss on the excluded point \( (-e_i, 1) \) is  \(\left( -w'(i) - 1 \right)^2 = (-2 - 1)^2 = 9\) while on \((e_i, 2)\) individual loss is \(0\), and hence $L_D(w') \geq L_D(w^\star) + \frac{9 + 0-9/2}{d}>L_D(w^\star)$. The case when \( D' \) includes \( (-e_i, 1) \) but excludes \( (e_i, 2) \)) is analyzed similarly.
(ii) If \( D' \) excludes both \( (e_i, 2) \) and \( (-e_i, 1) \) then $L_{D'}(w)$ does not depend on the $i$'th coordinate of $w$, hence, by the min-norm property \(w'(i) = 0\) and its individual losses on the excluded points \( (e_i, 2) \) and \( (-e_i, 1) \) are respectively $(0-2)^2=4$ and $(0-1)^2=1$. Thus,
$L_D(w') \geq L_D(w^\star) + \frac{5/2 - 9/4}{d} > L_D(w^\star)$.
\end{example}



\vspace{-9pt}\subsection{A General Result for Strictly Convex Losses}\label{sec:sco}
\vspace{-3pt}
We now present a general result for stochastic convex optimization (SCO) problems.
\vspace{-3pt}
\begin{mdframed}
\begin{theorem}[Stochastic Convex Optimization]\label{thm:sco}
Consider a SCO problem with a parameter space \(\mathcal{W} \subseteq \mathbb{R}^d\) that is closed and convex. Assume further that each loss function \(\ell_z\) for \(z \in \mathcal{Z}\) is strictly convex. Then, for any ERM \(A\), every dataset \(D\), and any \(n > d\):
\[
L^\star_D(n ; A) = L^\star_D.
\]
{Moreover, the strict convexity and \(n > d\) assumptions are necessary, as illustrated by Examples~\ref{ex:strict_convexity} and \ref{ex:requirement_n_greater_d}) in Section~\ref{app:examples}.}
\end{theorem}
\end{mdframed}
\vspace{-3pt}
The proof of this theorem follows directly from Proposition~\ref{prop:car}.
This result highlights the benefits of using regularization to induce strict convexity in optimization problems. By ensuring strict convexity, one can perform data selection with as few as  \(n = d+1\)  points, significantly reducing the selection budget. More broadly, regularization stabilizes the optimal solution by guaranteeing that it is determined by at most \(d+1\) examples—similar to the role of support vectors in SVMs.

\vspace{-10pt}\section{Future Research}\label{sec:open}

\vspace{-4pt}This work explored optimal data selection for natural algorithms in the contexts of classification and stochastic convex optimization. In classification, we focused on unweighted data selection, while in stochastic convex optimization, we considered weighted data selection.

\vspace{-4pt}\paragraph{Data Selection in Regression.}  
A natural extension of our study is to explore unweighted data selection for regression problems. For example, our analysis can be adapted to show that the results in Theorem~\ref{thm:linreg} for linear regression in the cases \(n = d\) and \(n < d\) also hold for unweighted data selection. Studying the case of \(n > d\) in the unweighted setting remains an open question. It would also be interesting to study weighted data selection for other ERMs, beyond the min-norm ERM addressed in Theorem~\ref{thm:linreg}.  

Additionally, the case \( d < n < 2d \) is not handled by Theorem~\ref{thm:linreg}, and it would be interesting to characterize the approximation guarantees in this regime.

\vspace{-4pt}\paragraph{Additive Approximation Guarantees.}
In this work, we focused on multiplicative approximation guarantees in the context of stochastic convex optimization, as these are natural for unbounded loss functions. Exploring additive approximation guarantees for compact SCO problems is an interesting direction for future research, particularly to understand how they depend on properties such as Lipschitz continuity, smoothness, strong convexity, and the diameter of the parameter space.

\vspace{-4pt}\paragraph{Non-Continuous ERMs in Classification.}
A natural direction for future research is to relax the continuity assumption in linear classification. Specifically, is it possible to construct an ERM for linear classification that achieves non-trivial guarantees with $n \leq d$? As shown in Figure~\ref{fig:non-continuous-erm}, non-continuous ERMs can leverage discontinuities in decision boundaries in the case $n = d$.

\vspace{-4pt}\paragraph{High-Dimensional Mean Estimation.}
The mean estimation problem analyzed in Theorem~\ref{t:mean} naturally extends to higher dimensions, formulated as a stochastic convex optimization problem with \(\mathcal{W} = \mathbb{R}^d\), \(\mathcal{Z} = \mathbb{R}^d\), and loss functions defined as \(\ell_z(h) = \|z - h\|^2\). An interesting question is how the results of Theorem~\ref{t:mean} generalize as the dimension \(d\) increases.

For every \( d \geq 1 \), the following bounds hold:
\[
\frac{2n}{2n-1} \leq \sup_{D \subseteq \mathbb{R}^d} \frac{L_D^\star(n)}{L_D^\star} \leq \frac{n+1}{n}.
\]
The lower bound is tight for \( d = 1 \), as established in Theorem~\ref{t:mean}. The upper bound becomes asymptotically tight as the dimension tends to infinity:
\[
\lim_{d \to \infty} \sup_{D \subseteq \mathbb{R}^d} \frac{L_D^\star(n)}{L_D^\star} = \frac{n+1}{n}.
\]
We prove these results is in Section~\ref{sec:p20p21}. 

An open question is to determine the exact tight bounds for fixed dimensions \(d = 2, 3, \ldots\). Another interesting direction is to study bounds for weighted data selection in this problem.

\section{Related Work}\label{sec:related}

While most of learning theory has predominantly followed a model-centric approach, focusing on designing and analyzing algorithms, there are a few subareas that can be viewed as more data-centric. One notable example is the study of sample compression schemes~\citep*{littlestone:86}, which aim to derive generalization bounds for algorithms whose outputs depend on a small subset of selected input examples. Another prominent example is active learning~\citep*{cohn:94, balcan:06, hanneke:fntml}, where the goal is to achieve effective learning by labeling as few examples as possible. A related variant studied by \citet*{pmlr-v97-dasgupta19a} considers a black-box active learning model, where the learner interacts with an oracle that returns hypotheses instead of labels. Similarly, \citet*{pmlr-v119-cicalese20a} develop an active learning algorithm with improved bounds. 

More recently, there has been a line of work in machine teaching and coresets that relates to our problem. A \emph{coreset} is a small weighted subset of a dataset that approximately preserves some key property (e.g., loss, margin, or clustering cost) of the full dataset. This allows for more efficient computation without significant loss in accuracy. Several works have explored different approaches to constructing coresets. \citet*{NEURIPS2019_475fbefa} propose a new algorithm for computing Carath\'eodory sets with improved efficiency, reducing the complexity to \( O(nd) \). In the context of approximation, \citet*{Feldman20} show that low-rank matrix approximations can estimate distances to compact sets spanned by \( k \) vectors in \( \mathbb{R}^d \) up to a \((1+\epsilon)\) factor. Expanding this direction, \citet*{Feldman24} explore a learning-based approach to constructing coresets, while \citet*{Feldman23} and \citet*{borsos2020coresets} study coreset methods tailored to neural networks. 

Similar questions have also been explored within machine teaching. For instance, \citet{Ma18a} study a setting where a teacher, who knows the target concept, selects a small subdataset from a sample drawn IID from the population to train the learner. This setup is closely related to our problem, with the key difference that the selection process in our setting operates directly on $D$ rather than an intermediate sample. Their analysis focuses on the maximum likelihood estimator for the mean of a Gaussian and the large-margin classifier in one dimension.

\section{Proofs}\label{proofs}

\subsection{Carath\'eodory's Theorem for Convex Functions}

\begin{proposition*}[Carath\'eodory's Theorem for Convex Functions]
Let \( K \subseteq \mathbb{R}^d \) be convex and closed, and let \( f_1, \dots, f_n : K \to \mathbb{R} \) be strictly convex functions. Then, for any \( g \in \mathtt{conv}(f_1, \dots, f_n) \), there exist indices \( i_1 \leq \ldots \leq i_{d+1} \) and a function \( g' \in \mathtt{conv}(f_{i_1}, \dots, f_{i_{d+1}}) \) such that:
\begin{enumerate}
    \item \(\arg\min_{x \in K} g'(x) = \arg\min_{x \in K} g(x)\), and
    \item \(\min_{x \in K} g'(x) \leq \min_{x \in K} g(x)\).
\end{enumerate}
\end{proposition*}

Before we begin the proof, we recall some basic concepts from convex function analysis.
A function \( f : \mathbb{R}^d \to \mathbb{R} \) is \emph{strictly convex} if, for all \( x, y \in \mathbb{R}^d \) with \( x \neq y \), and for any \( \lambda \in (0, 1) \), it holds that
$f(\lambda x + (1-\lambda)y) < \lambda f(x) + (1-\lambda)f(y).$
This definition is like the definition of convex functions, but with a strict inequality replacing the non-strict one. A key property of strictly convex functions is that they have a unique global minimum, as the strict inequality prevents any other point from achieving the same minimum value.
A vector \( g \in \mathbb{R}^d \) is a \emph{subgradient} of a convex function \( f : \mathbb{R}^d \to \mathbb{R} \) at a point \( \xi \in \mathbb{R}^d \) if:
\[
f(y) \geq f(\xi) + \langle g, y - \xi \rangle \quad \text{for all } y \in \mathbb{R}^d.
\]
The set of all subgradients of \( f \) at \( \xi \) is called the \emph{subdifferential} and is denoted \( \partial f(\xi) \).
The \emph{Minkowski sum} of two sets \( A, B \subseteq \mathbb{R}^d \) is defined as:
\[
A + B = \{a + b \mid a \in A, b \in B\}.
\]
We will use the following fundamental result:
\begin{proposition}[Additivity of Subgradients \texorpdfstring{\citep[p. 223]{rockafellar1997convex}}{}] \label{prop:addsub}
Let \( K \subseteq \mathbb{R}^d \) be convex and closed, and let \( f, g : K \to \mathbb{R} \) be convex functions. For any \( \xi \in K \), it holds that:
\[
\partial \left(f + g\right)(\xi) = \partial f(\xi) + \partial g(\xi).
\]
\end{proposition}

This result generalizes the familiar fact that the gradient of a sum is the sum of the gradients, extending it to convex functions that are not necessarily differentiable.
We will also use the following simple fact:
\begin{lemma}\label{lem:sub_max}
Let \( f : \mathbb{R}^d \to \mathbb{R} \) be strictly convex. Then \( 0 \in \partial f(\xi) \) if and only if \( \xi \) is the unique global minimum of \( f \).
\end{lemma}
\begin{proof}
    Suppose \(0 \in \partial f(\xi)\). By the definition of the subgradient:
    \[
    f(y) \geq f(\xi) + \langle 0, y - \xi \rangle = f(\xi), \quad \forall y \in \mathbb{R}^d.
    \]
    Thus, \(f(\xi)\) is a global minimum. Since \(f\) is strictly convex, the global minimum is unique, so \(\xi\) is the unique minimizer. Conversely, if \(\xi\) is the unique global minimum, then \(f(y) \geq f(\xi)\) for all \(y \in \mathbb{R}^d\). By the subgradient definition, \(0 \in \partial f(\xi)\). This proves the claim.
\end{proof}
\begin{proof}[Proof of Proposition~\ref{prop:car}]
The proof follows a similar inductive sparsification process for the convex combination as in the classical Carathéodory Theorem (see, e.g.~\citep{Matousek2002} page 6 Theorem 1.2.3, originally in ~\citep{Carathéodory1907}). The key difference is the need for a more careful selection during sparsification to ensure the second item (the inequality between the minimum values) is preserved.

We prove the result by induction on \( n \), the number of functions in the convex combination. The base case \( n \leq d+1 \) is trivial since \( g' = g \) satisfies the conclusion.  
For \( n > d+1 \), write \( g \in \mathtt{conv}(f_1, \dots, f_n) \) as:
\[
g = \sum_{i=1}^n \lambda_i f_i, \quad \text{where } \lambda_i \geq 0 \text{ and } \sum_{i=1}^n \lambda_i = 1.
\]
Further assume that $\lambda_i>0$ for all $i$; if this is not the case, and $\lambda_i=0$
for some $i$, then we can remove $f_i$ from the convex combination and apply induction.
Let \( \xi = \arg\min_{x \in K} g(x) \). By Lemma~\ref{lem:sub_max}, \( 0 \in \partial g(\xi) \), so:
\[
0 = \sum_{i=1}^n \lambda_i y_i, \quad \text{with } y_i \in \partial f_i(\xi). \tag{Proposition~\ref{prop:addsub}}
\]
Since \( n > d+1 \), the vectors \( \{y_n - y_1, \dots, y_n - y_{n-1}\} \subset \mathbb{R}^d \) are linearly dependent. Thus, there exist coefficients \( \beta_1, \dots, \beta_n \) such that $\sum_{i=1}^n \beta_i y_i = 0$, where $\sum_{i=1}^n \beta_i = 0$.
Define \( \lambda_i(t) = \lambda_i + t \beta_i \) and the corresponding functions:
\[
g(x,t) = \sum_{i=1}^n \lambda_i(t) f_i(x).
\]
Notice that \( \sum_{i=1}^n \lambda_i(t) = 1 \) for all $t$ and, since $\lambda_i(0)=\lambda_i>0$, it follows that for a small enough \( \lvert t\rvert>0 \), \( \lambda_i(t) \geq 0 \). Let \( T \) be the maximal interval containing \( t = 0 \) such that \( \lambda_i(t) \geq 0 \) for all \( i \). At least one \( \lambda_i(t) \) becomes zero at the endpoints \( t_0, t_1 \) of \( T \).
Since for every fixed $x$, \( g(x,t) \) is linear in \( t \), its minimum over \( K \) occurs at \( t_0 \) or \( t_1 \). Let \( t^* \in \{t_0, t_1\} \) minimize \( g(\xi,t) \). At \( t = t^* \), the corresponding function:
\[
g_{t^\star}(x) = g(x, t^\star) = \sum_{i=1}^n \lambda_i(t^*) f_i(x)
\]
involves fewer than \( n \) strictly positive coefficients and satisfies the conditions:
\[
\arg\min_{x \in K} g_{t^\star}(x) = \arg\min_{x \in K} g(x), \quad \min_{x \in K} g_{t^\star}(x) \leq \min_{x \in K} g(x).
\]
Applying induction on $g_{t^\star}$ finishes the proof of Proposition.
\end{proof}

\subsection{Mean Estimation: Proof of Theorem~\ref{t:mean}}

\begin{theorem*}[Theorem~\ref{t:mean} Restatement]
    For every \( n \geq 1 \),
\[
\sup_{D \subseteq \mathbb{R}} \frac{L_D^\star(n)}{L_D^\star} = \frac{2n}{2n-1},
\]
where \( D \) ranges over all finite multisets of \( \mathbb{R} \). Above, the ratio \(\frac{L^\star_D(n)}{L^\star_D}\) is defined to be \(1\) when \(L_D^\star(n)= L_D^\star=0\). If only \(L_D^\star=0\), the ratio is defined as \(\infty\).
\end{theorem*}

It is convenient to allow the dataset $D$ to be any finitely supported distribution, rather than a finite multiset. Importantly, the data-selection process remains unchanged in this formulation: the data selector can choose any sequence of \(n\) points from the support of \(D\) and is evaluated based on the loss incurred by the average of these \(n\) points. This formulation is equivalent to the original problem, as finite multisets correspond to rational distributions, and rational distributions are dense in the set of all finitely supported distributions.

Throughout the proof, we rely on the following lemma, which can be derived through a straightforward calculation:

\begin{lemma}\label{lem:13}
    Let \(D\) be a distribution supported on \( \mathbb{R}^d \). 
    For any point \( h \in \mathbb{R}^d \), the squared loss function \( L_D(h) \) satisfies:
    \[
    L_D(h) = L_D^\star + \|h - \mu_D\|^2,
    \]
    where \( \mu_D = \Ex_{z\sim D}z \) is the mean of \( D \).
\end{lemma}

\paragraph{Lower Bound.} \label{t:mean:lower}
We begin by proving the lower bound using the simple dataset \(D\) consisting of \(2n-1\) copies of \(0\) and a single copy of \(1\). A direct calculation yields the optimal loss:
\[
L^\star_D = \frac{2n - 1}{2n}\Bigl(\frac{1}{2n} - 0\Bigr)^2 + \frac{1}{2n}\Bigl(\frac{1}{2n} - 1\Bigr)^2 = \frac{2n-1}{(2n)^2}.
\]
If we select only \(n\) points from \(D\), the achievable averages are restricted to the set \(\{0, k/n : k < n\}\). Applying Lemma~\ref{lem:13}, we compute:
\[
\frac{L^\star_D(n)}{L^\star_D} = \frac{L^\star_D + (1/2n)^2}{L^\star_D} = \frac{2n}{2n-1},
\]
and hence $\sup_{D \subseteq \mathbb{R}} \frac{L_D^\star(n)}{L_D^\star} \geq \frac{2n}{2n-1}$.
\paragraph{Upper Bound.}  
We now turn to the upper bound, which requires more work. Following the proof outline, we proceed in two steps:  
(i) use Proposition~\ref{prop:car} to reduce the problem to the case where \(D\) consists of only two points, and  
(ii) explicitly analyze this reduced case.

\paragraph{Step 1: Reduction to two points.}
Let \( D \) be a (finitely supported) distribution over \( \mathbb{R} \).
The case of $L^\star_D=0$ is trivial (in this case $D$ is a dirac distribution and $\frac{L_D^\star(n)}{L_D^\star} = 1$), and hence we assume that $L^\star_D > 0$.
We will prove that there exist a distribution $D'$ supported on two points $z_1,z_2\in \mathtt{supp}(D)$ such that for all $n$
\[\frac{L^\star_D(n)}{L^\star_D} \leq  \frac{L^\star_{D'}(n)}{L^\star_{D'}}.\]
For each \( z_i \in \mathtt{supp}(D) \), define \( f_i(x) = (x - z_i)^2 \). Since \( f_i(x) \) is strictly convex, the loss function \( L_D(x) \) is expressed as:
\[
L_D(x) = \mathbb{E}_{z \sim D}[(x - z)^2] = \sum_{z_i \in S} p_i (x - z_i)^2,
\]
where \( p_i \) are the probabilities associated with \( z_i \) under \( D \).
By Proposition~\ref{prop:car}, the strictly convex function \( L_D(x) \), being a convex combination of strictly convex functions, can be approximated by a function \( g' \), which is a convex combination of at most two functions \( f_1(x) \) and \( f_2(x) \). Without loss of generality, let these functions correspond to points \( z_1 \) and \( z_2 \). Define a new distribution \( D' \) supported only on \( \{z_1, z_2\} \), with probabilities \( p_1 \) and \( p_2 \) corresponding to the coefficients in the convex combination defining \( g' \). Then:
\[
g'(x) = L_{D'}(x) = p_1 (x - z_1)^2 + p_2 (x - z_2)^2.
\]
By the second property of Proposition~\ref{prop:car}, \( L_{D'}^\star \leq L_D^\star \), and hence:
\begin{align*}
\frac{L_D^\star(n)}{L_D^\star} &= 1 + \frac{\min_{z_1,\ldots,z_n\in \mathtt{supp}(D)} \lvert\mu_D - \bar z \rvert^2}{L_D^\star} \tag{Lemma~\ref{lem:13}}\\
&\leq 
1 + \frac{\min_{z_1,\ldots,z_n\in \mathtt{supp}(D')}\lvert\mu_{D'} - \bar z \rvert^2}{L_{D'}^\star}
\tag{$\mathtt{supp}(D')\subseteq \mathtt{supp}(D)$, $L_{D'}^\star\leq L_{D}^\star$, $\mu_D = \mu_{D'}$ }\\
&= \frac{L_{D'}^\star(n)}{L_{D'}^\star},
\end{align*}
where $\bar z$ above is the average $\bar z = (z_1,\ldots z_n)/n$, and $\mu_{D},\mu_{D'}$ are the means of $D$ and $D'$ respectively.
This reduction allows us to focus on distributions supported on exactly two points, \(\mathtt{supp}(D) = \{z_1, z_2\}\). Since translating or rescaling the support does not affect the ratio \(\frac{L_D^\star(n)}{L_D^\star}\), we can assume without loss of generality that \(\mathbb{E}[D] = 0\), \(\mathrm{Var}(D) = 1\), and \(z_1 < 0 < z_2\) with \(|z_2| \geq |z_1|\).

\paragraph{Step 2: Analysis of the two-point case.}
We use the following lemma:

\begin{lemma}\label{lem:14}
    Let \( D \) be a distribution supported on two points \( z_1, z_2 \in \mathbb{R} \) with expectation \( \mathbb{E}_D[X] = 0 \) and variance \( \mathrm{Var}_D[X] \leq 1 \). Then, \( \lvert z_1 \cdot z_2 \rvert \leq 1 \).
\end{lemma}

We will first use Lemma~\ref{lem:14} to complete the proof of Theorem~\ref{t:mean}, deferring the proof of the lemma to the end. By the lemma, \( \lvert z_1 \cdot z_2 \rvert \leq 1 \). Since $\mathrm{Var}_D[X] \leq 1$, it follows that $z_1 \leq 1$. Now, consider two cases: (i) if \( z_1 \leq \frac{1}{\sqrt{2n-1}} \), then by selecting only $z_1$, we have
    $L_D^\star(n) \leq 1 + \lvert \mu_D - z_1 \rvert^2 = 1 + \lvert z_1 \rvert^2 \leq 1 + \frac{1}{2n-1}$,
    which achieves the desired bound.
    (ii) Else, \( z_1 > \frac{1}{\sqrt{2n-1}} \). Consider the function \( z_1 + \frac{1}{z_1} \), and observe that it is decreasing for \( 0 < z_1 \leq 1 \). Therefore,
    \[
    z_1 + \frac{1}{z_1} \leq \frac{1}{\sqrt{2n-1}} + \sqrt{2n-1} = \frac{2n}{\sqrt{2n-1}}.
    \]
    By Lemma~\ref{lem:14}, the length of the interval \( [z_2, z_1] \) is bounded by \( z_1 + \frac{1}{z_1} \leq \frac{2n}{\sqrt{2n-1}} \). The averages of all possible selections of \(n\) points from \(D\) form a uniform grid along this interval, with the spacing between consecutive points on the grid at most \( \frac{2}{\sqrt{2n-1}} \). Since \(\mu_D = 0\) lies within the interval, its distance to the nearest point on the grid is at most \( \frac{1}{\sqrt{2n-1}} \). Thus, using Lemma~\ref{lem:13}, we have:
\[
L_D^\star(n) \leq 1 + \left( \frac{1}{\sqrt{2n-1}} \right)^2 = \frac{2n}{2n-1}.
\]
To conclude, in both cases, we achieve the bound
$L_D^\star(n) \leq \frac{2n}{2n-1}$, which completes the proof.

\begin{proof}[Proof of Lemma~\ref{lem:14}]
    The loss function is given by \( L_D(x) = \mathbb{E}_D[(x - X)^2] \), which can be expressed as:
    \[
    L_D(x) = \mathrm{Var}_D[X] + (x - \mu_D)^2,
    \]
    where \( \mu_D = \mathbb{E}_D[X] = 0 \). Thus, \( L_D(x) = \mathrm{Var}_D[X] + x^2 \).

    Assume towards contradiction that \( \lvert z_1 \cdot z_2 \rvert > 1 \), and without loss of generality, assume \( z_1 > 0 \). Then, the interval \( \left[-\frac{1}{z_1}, z_1\right] \) does not contain \( z_2 \). 
    Let $c = \frac{z_1 - \frac{1}{z_1}}{2}$ denote the center of this interval
    Any point in the support of \( D \) has distance at least \( d = \frac{z_1 + \frac{1}{z_1}}{2} \) from \( c \). Substituting into \( L_D(c) \), we have $L_D(c) = \mathrm{Var}_D[X] + c^2 > d^2$.
    Therefore,
    \begin{align*}
    \mathrm{Var}_D[X] &> d^2 - c^2 \\
                      &= \frac{z_1^2 + \frac{1}{z_1^2} + 2}{4} - \frac{z_1^2 + \frac{1}{z_1^2} - 2}{4}\\
                      &= 1,
    \end{align*}
    which is a contradiction.
    Therefore, \( \mathrm{Var}_D[X] \leq 1 \) implies \( \lvert z_1 \cdot z_2 \rvert \leq 1 \) as stated.
\end{proof}

\subsection{Classification}\label{sec:proofsbinary}

\subsubsection{Proof of Theorem~\ref{thm:linbinary}}
\begin{theorem*}[Theorem~\ref{thm:linbinary} Restatement]
Let \(A^\star\) denote the max-margin algorithm. Then,
\[
\sup_{D \in \mathtt{Real}(\mathcal{H}_d)} L^\star_D(n ; A^\star) =
\begin{cases}
0 & \text{if } n > d, \\
\frac{1}{2} & \text{if } n \leq d,
\end{cases}
\]
where $D$ ranges over all realizable datasets.
Furthermore, \(A^\star\) is optimal in the sense that for every continuous ERM \(A\) (and even for any continuous proper learner),
\[
(\forall n \leq d): \quad \sup_{D \in \mathtt{Real}(\mathcal{H}_d)} L^\star_D(n ; A) \geq \frac{1}{2}.
\]
\end{theorem*}
\begin{proof}
The fact that the max-margin algorithm \(A^\star\) satisfies \(\sup_{D \in \mathtt{Real}(\mathcal{H}_d)} L^\star_D(n ; A^\star) = 0\) whenever \(n > d\) follows from \citep*{vapnik:74}; see also Appendix A of \citet*{LongL20}. It remains to show that when \(n \leq d\), \(\sup_{D \in \mathtt{Real}(\mathcal{H}_d)} L^\star_D(n ; A^\star) \leq \frac{1}{2}\) for the max-margin algorithm, and that \(\sup_{D \in \mathtt{Real}(\mathcal{H}_d)} L^\star_D(n ; A) \geq \frac{1}{2}\) for any continuous proper learner \(A\).

The first part follows from the observation that when presented with only positively labeled examples or only negatively labeled examples, 
the max-margin algorithm outputs the constant \(+1\) or \(-1\) hypothesis, respectively. 
Thus, by selecting a single example from \(D\) whose label matches the majority label, it follows that
$\sup_{D \in \mathtt{Real}(\mathcal{H}_d)} L^\star_D(n ; A^\star) \leq \frac{1}{2}.$

We now turn to the more challenging task of proving that for any continuous proper learner \( A \), we have \(\sup_{D \in \mathtt{Real}(\mathcal{H}_d)} L^\star_D(n ; A) \geq \frac{1}{2}\) whenever \( n \leq d \).
Let \( A \) be a continuous proper learner and let $\eta >0$. Using the probabilistic method, we construct a dataset \( D \) such that $L^\star_D(n ; A) \geq \frac{1}{2} - \eta$ as follows. For a large enough $N=N(\eta)$ (to be specified later) let
\( (x_1, y_1), \ldots, (x_N, y_N) \) be a sequence of IID examples drawn from the following distribution over \( \mathbb{R}^d \times \{\pm 1\} \): each label \( y_i \) is chosen uniformly at random from \( \{\pm 1\} \), and each feature vector \( x_i = (x_i(1), \ldots, x_i(d)) \) is sampled independently such that the last coordinate satisfies \( x_i(d) = 0 \) and each of the first \( d-1 \) coordinates \( x_i(j) \) is sampled uniformly from \( [0, 1] \) and independently of the others.
We claim that, with probability \(>0\), the dataset \(D\) satisfies the following properties:  
\begin{enumerate}
    \item Every subset of \(D\) containing at most \(d\) points is realizable by \(\mathcal{H}_{d}\).
    \item Every halfspace \(h \in \mathcal{H}_{d}\) has a classification error on \(D\) of at least \(\frac{1}{2} - \eta\).
\end{enumerate}
Let us begin with the first property. We claim that the first property holds with probability $1$. 
Let \( D' \) denote the \( (d-1) \)-dimensional dataset obtained by omitting the last coordinate from each point \( x_i \) in \( D \) (noting that this coordinate is \( 0 \) for all \( x_i \)'s). Specifically, \( D' = \{(x_i', y_i)\}_{i=1}^N \), where \( x_i' = (x_i(1), \ldots, x_i(d-1)) \). 
Let \(1 \leq i_1 < \ldots < i_d \leq N\) be \(d\) distinct indices.
By the hyperplane separation theorem, we have:
\begin{align*}
    \Pr[\{(x_{i_j}, y_{i_j})\}_{j=1}^d \text{ is not realizable by } \mathcal{H}_d] 
    &= \Pr\Bigl[\mathtt{conv}\{x_{i_j} : y_{i_j} = +1\} \cap \mathtt{conv}\{x_{i_j} : y_{i_j} = -1\} \neq \emptyset\Bigr] \\
    &\leq \Pr\Bigl[\{x_{i_j}'\}_{j=1}^d \text{ are affinely dependent in } \mathbb{R}^{d-1}\Bigr] \\
    &= 0.
\end{align*}
The inequality holds because if the convex hulls intersect, there exists a point \( x \) such that:
\[
x = \sum_{j : y_{i_j}=+1} \alpha_j x_{i_j} = \sum_{j : y_{i_j}=-1} \beta_j x_{i_j},
\]
where \(\alpha_j, \beta_j \geq 0\) and \(\sum_{j : y_{i_j}=+1} \alpha_j = \sum_{j : y_{i_j}=-1} \beta_j = 1\). Subtracting these and omitting the last coordinate gives:
\[
0 = \sum_{j : y_{i_j}=+1} \alpha_j x_{i_j}' - \sum_{j : y_{i_j}=-1} \beta_j x_{i_j}',
\]
which implies a non-trivial affine dependency among the \( x_{i_j}' \)'s. 
Since the \( x_i' \)'s are independent uniform samples from the continuous cube \([0,1]^{d-1} \subseteq \mathbb{R}^{d-1}\), with probability \(1\), any \(d\) of them are affinely independent~\citep[p.~3, General position]{Matousek2002}. Thus,
\[
\Pr[\{(x_{i_j}, y_{i_j})\}_{j=1}^d \text{ is not realizable by } \mathcal{H}_d] \leq \Pr\big[\{x_{i_j}'\}_{j=1}^d \text{ are affinely dependent in } \mathbb{R}^{d-1}\big] = 0.
\]
Since this holds for any choice of \(d\) distinct indices \(i_1, \ldots, i_d\), it also holds simultaneously for all such choices, as a finite union of measure zero events has measure zero. This establishes that the first property holds with probability \(1\).

We next show that the second property holds with high probability. Specifically, we prove:  
\[
\Pr\left(\exists h \in \mathcal{H}_{d} \,|\, L_D(h) < \frac{1}{2} - \eta\right) \to 0 \quad \text{as } N \to \infty.
\]
Let \(P\) denote the distribution over \(\mathbb{R}^d \times \{\pm 1\}\) from which \(D\) is sampled.  
Since the labels \(y\) are sampled uniformly from \(\{\pm 1\}\) and independently of \(x\), we have \(L_P(h) = \Pr_{(x,y)\sim P}[h(x)\neq y] = 1/2\). Therefore,  
\[
\Pr_{D \sim P^N}\left[\exists h \in \mathcal{H}_{d} \,|\, L_D(h) < \frac{1}{2} - \eta\right] \leq 
\Pr_{D \sim P^N}\left[\exists h \in \mathcal{H}_{d} \,|\, \lvert L_D(h) - L_P(h) \rvert > \eta\right],
\]
where the inequality holds because \(L_P(h) = 1/2\) for all \(h \in \mathcal{H}_d\).  
Since \(\mathcal{H}_d\) has finite VC dimension (specifically, \(\mathtt{vc}(\mathcal{H}_d) = d+1\)), uniform convergence guarantees that:  
\[
\Pr_{D \sim P^N}\left[\exists h \in \mathcal{H}_{d} \,|\, \lvert L_D(h) - L_P(h) \rvert > \eta\right] \to 0 \quad \text{as } N \to \infty,
\]
as desired.

We conclude that there exists a finite dataset \( D \) satisfying both properties:
\begin{enumerate}
    \item Every sub-dataset of \( D \) containing at most \( d \) points is realizable by \(\mathcal{H}_{d}\). In fact, every such sub-dataset has feature vectors that are affinely independent.
    \item Every halfspace \( h \in \mathcal{H}_{d} \) has a classification error on \( D \) of at least \( \frac{1}{2} - \eta \).
\end{enumerate}

For each \( \epsilon > 0 \), define a dataset \( D_\epsilon = \{(x_i + y_i \cdot \epsilon \cdot e_d, y_i)\}_{i=1}^N \), where \( e_d \) is the unit vector in the \(d\)-th direction. That is, the vector \( x_i + y_i \cdot \epsilon \cdot e_d \) has the same first \( d-1 \) coordinates as \( x_i \), while its last coordinate is shifted by \( y_i \cdot \epsilon \). This transformation "lifts" the points with label \( +1 \) and "lowers" the points with label \( -1 \). Observe that for every \( \epsilon > 0 \), the dataset \( D_\epsilon \) is realizable by the halfspace \( (x(1), \ldots, x(d)) \mapsto \mathtt{sign}(x(d)) \). Note also that when \( \epsilon = 0 \), we recover \( D_0 = D \).

Now, let \( D' \subseteq D \) be a sub-dataset of size \( n \), and let \( D'_\epsilon \subseteq D_\epsilon \) be the corresponding sub-dataset of \( D_\epsilon \). Denote by \( (w, b) = A(D') \) and \( (w_\epsilon, b_\epsilon) = A(D'_\epsilon) \) the parameters of the halfspace output by \( A \) when applied to \( D' \) and \( D'_\epsilon \), respectively. 

For any \( x \in \mathbb{R}^d \) such that \( w \cdot x + b \neq 0 \), continuity implies that for every sufficiently small \( \epsilon = \epsilon(x) > 0 \), we have:
\[
w_\epsilon\cdot x +b_\epsilon \neq 0 \quad\text{ and } \quad
\mathtt{sign}(w \cdot x + b) = \mathtt{sign}(w_\epsilon \cdot x + b_\epsilon).
\]
By the first property of \( D \), there are at most \( d \) feature vectors \( x_i \) in \( D \) such that \( w \cdot x_i + b = 0 \). Therefore, for every sufficiently small \( \epsilon = \epsilon(D') > 0 \), we obtain:
\[
L_{D_\epsilon}(A(D'_\epsilon)) \geq \frac{1}{2} - \eta - \frac{d}{N}.
\]
Finally, by choosing \( \epsilon^\star > 0 \) smaller than \( \epsilon(D') > 0 \) for all sub-datasets \( D' \subseteq D \) of size \( n \), we construct a realizable dataset \( D_{\epsilon^\star} \) such that:
\[
L_{D_{\epsilon^\star}}(n; A) \geq \frac{1}{2} - \eta - \frac{d}{N}.
\]
The proof concludes by noting that \( \eta \) can be made arbitrarily small and \( N \) arbitrarily large.

\end{proof}

\subsubsection{Proof of Theorem~\ref{thm:binary}}

\begin{theorem*}[Restatement of Theorem~\ref{thm:binary}]
Every hypothesis class \(\mathcal{H} \subseteq \{\pm 1\}^X\) satisfies exactly one of the following:
\begin{enumerate}
    \item \(\mathtt{R}_{\mathcal{H}}^\star(n) = 1\), for all \(n \in \mathbb{N}\). \hfill (Trivial Rate)
    \item \(\frac{C_1}{n} \leq \mathtt{R}_{\mathcal{H}}^\star(n) \leq  \frac{C_2 \cdot\log n}{n}\), for all \(n \in \mathbb{N}\). Here \( C_1=C_1(\H),C_2=C_2(\H)\) are positive constants that depend on \(\mathcal{H}\) (but not on $n$). \hfill (Linear Rate)
    \item \(\mathtt{R}_{\mathcal{H}}^\star(n) = 0\), for all sufficiently large \(n \geq n_0(\mathcal{H})\). \hfill (Zero Rate)
\end{enumerate}
\end{theorem*}
We will rely on basic results from the theory of $\varepsilon$-nets for families with bounded VC dimension and star number. We begin with some useful definitions.

\begin{definition}[$\varepsilon$-nets] \label{def:eps:set}
    Let $X$ be a set, let $\mathcal{F}\subseteq 2^X$ be a family of subsets of $X$, and let \(P\)
    be a distribution over $X$. An \(\varepsilon\)-net for \(\mathcal{F}\) with respect to \(P\) is a set \(S \subseteq X\) such that for every \(F \in \mathcal{F}\),
    \[
        P(F) > \varepsilon \quad \Rightarrow \quad F \cap S \neq \emptyset.
    \]
    In other words, \(S\) intersects every set in \(\mathcal{F}\) that has probability measure greater than \(\varepsilon\).
\end{definition}

\begin{definition}[Star Number for Set Families] \label{def:star:set} 
Let \(X\) be a set and let $\mathcal{F}$ be a family of subsets of $X$.
The star number of $\mathcal{F}$ ,  denoted $s_\mathcal{F}$,  is the largest $n \in \mathbb{N}$  such that there exist points $x_1, \dots, x_n \in \mathcal{X}$  satisfying
\[
(\forall i \in \{0, \dots, n\})(\exists S_i \in \mathcal{F}):  S_i \cap \{x_1, \dots, x_n \} = \{x_i\}.
\]
If no such largest $n$ exists, define $\mathfrak{s}_\mathcal{F}=\infty$.
\end{definition}

\begin{definition}[Star Number for Hypothesis Classes] \label{def:star:hc} Let \(\mathcal{X}\) be a set and let hypothesis class $\mathcal{H}\subseteq \{\pm1\}^\mathcal{X}$ and let $h\in\mathcal{H}$. The star number of $\mathcal{H}$  centered at $h$,  denoted $s_h = s_h(\mathcal{H})$,  is the largest $n \in \mathbb{N}$  such that there exist points $x_1, \dots, x_n \in \mathcal{X}$  satisfying
\[
\forall i \in \{0, \dots, n\}, \ \exists h_i \in \mathcal{H} \text{ such that } \forall j \in \{1, \dots, n\}, \ h_i(x_j) = h(x_j) \iff j \neq i.
\]
 If no such largest $n$  exists, define $s_h = \infty$. The star number of $\mathcal{H}$  is defined as
\(\mathfrak{s}_\mathcal{H} = \sup_{h \in \mathcal{H}} s_h.\)
\end{definition}
We will rely on the following upper bounds on $\varepsilon$-nets in terms of the VC dimension and star number:
\begin{theorem}[$\varepsilon$-nets \citep*{haussler:87,hanneke:15b}]\label{thm:epsnet}
Let $\mathcal{F}\subseteq 2^X$ be a family of sets. Then,
\begin{enumerate}
    \item If $\mathfrak{s}_\mathcal{F} <\infty$ then for every distribution $P$ over $X$ and every $\varepsilon>0$,
    there exists an $\varepsilon$-net for $\mathcal{F}$ with respect to $P$ of size at most $\mathfrak{s}_\mathcal{F}$.
    \item If  \(\mathtt{vc}(\mathcal{F}) < \infty\) then for every distribution $P$ over $X$ and every $\varepsilon>0$,
    there exists an $\varepsilon$-net for $\mathcal{F}$ with respect to $P$ of size at most $O(\frac{\mathtt{vc}(\mathcal{F})\log(1/\varepsilon)}{\varepsilon})$, where the big oh conceals a multiplicative universal numerical constant.
\end{enumerate}
\end{theorem}

\noindent To establish Theorem~\ref{thm:binary}, we consider 4 cases:

\begin{enumerate}
    \item If \(\mathfrak{s}_\mathcal{H} < \infty\), then $\mathtt{R}_{\mathcal{H}}^\star(n)=0$ for every $n\geq \mathfrak{s}_\mathcal{H}$.
    
    \item If \(\mathfrak{s}_\mathcal{H} = \infty\), we show that \(\mathtt{R}_{\mathcal{H}}^\star(n)\) is lower bounded by \(\frac{1}{n+1}\) for all $n$.
    
    \item If the VC dimension of \(\mathcal{H}\) is finite, i.e., \(\mathtt{vc}(\mathcal{H}) = d < \infty\) then \(\mathtt{R}_{\mathcal{H}}^\star(n)\leq O(\frac{d \log n}{n})\) for all $n$.
    
    \item If \(\mathtt{vc}(\mathcal{H}) = \infty\), we prove that \(\mathtt{R}_{\mathcal{H}}^\star(n)=1\) for all $n$.
\end{enumerate}
These statements together establish the proof of Theorem~\ref{thm:binary}. 

\paragraph{Case 1: \(\mathfrak{s}(\mathcal{H}) < \infty\).} 
Given a dataset \(D\), consider a maximal realizable subdataset \(D' \subseteq D\). Noting that the optimal loss on \(D\) satisfies 
\[L^\star_D = \min_{h\in\mathcal{H}}L_D(h) =  \frac{\lvert D\rvert - \lvert D'\rvert}{\lvert D\rvert}.\]
Consider the following family of datasets:
\[
\operatorname{err}(\mathcal{H} \mid D') = \bigl\{\operatorname{err}(h \mid D') : h \in \mathcal{H} \bigr\}, \] 
where $\operatorname{err}(h \mid D') = \{(x,y) \in D' : h(x) \neq y\}$ is the subdataset of $D'$ which
$h$ classifies incorrectly.
Notice that the star number of \(\operatorname{err}(\mathcal{H} \mid D')\) is at most \(\mathfrak{s}_\mathcal{H}\). Let $P$ be the uniform distribution over $D'$.
By Theorem~\ref{thm:epsnet}, for every \(\varepsilon > 0\), there exists an \(\varepsilon\)-net \(D'_\varepsilon\subseteq D'\) for \(\operatorname{err}(\mathcal{H} \mid D')\) with respect to $P$ of size at most \(\mathfrak{s}_\mathcal{H}\).  Let $\mathcal{A}$ be an ERM for $\mathcal{H}$ and let $h=\mathcal{A}(D'_\varepsilon)$ denote the output hypothesis of $\mathcal{A}$ when given the $\varepsilon$-net as input.
Thus, by the $\varepsilon$-net property it follows that $L_{D'}(h)\leq \varepsilon$. Consequently, the total error on the full dataset \(D\) is bounded by \(L^\star_D + \varepsilon.\) Since this holds for any \(\varepsilon > 0\), we obtain \(\mathtt{R}_{\mathcal{H}}^\star(n, A) < \varepsilon\), implying that \(\mathtt{R}_{\mathcal{H}}^\star(n, A) = 0\), thus establishing the claim.

\paragraph{Case 2: \(\mathfrak{s}_\mathcal{H} = \infty\).} 
Given \(n\), select a dataset \(D\) that forms the center of a star-set of size \(n+1\). Such a dataset is realizable. Now, consider an ERM  that, whenever possible, avoids selecting the center hypothesis—specifically, when presented with any proper subdataset of \(D\), it chooses a hypothesis that is inconsistent with at least one point in \(D\). Consequently, that learner will misclassify at least one of the \(n+1\) points, leading to an error rate of at least \(\frac{1}{n+1}\). This establishes the desired lower bound:
\[
\mathtt{R}_{\mathcal{H}}^\star(n) \geq \frac{1}{n+1}.
\]

\paragraph{Case 3: \(\mathtt{vc}(\mathcal{H}) = d < \infty\).} 
The argument here is identical to the argument in Case 1. The only difference is that we invoke the 2nd Item of Theorem~\ref{thm:epsnet} which applies to the VC dimension. Specifically, given a selection budget of $n$ points, we obtain an $\varepsilon(n)$-net $D'_\varepsilon\subseteq D'$ for
\[
\varepsilon(n) = O\left(\frac{d \log (n/d)}{n} + \frac{\log (1/\delta)}{n} \right).
\]
such that applying any ERM on $D'_\varepsilon\subseteq D'$ yields an hypothesis whose loss is at most
$L^\star_D + \varepsilon(n)$.

Using this bound, we obtain an upper bound of the form \(\mathtt{R}_{\mathcal{H}}^\star(n) = O(\frac{d \log n}{n})\), establishing the linear rate.

\paragraph{Case 4: \(\mathtt{vc}(\mathcal{H}) = \infty\).} 
Following a similar approach to the \(\mathfrak{s}(\mathcal{H}) = \infty\) case, we now leverage the fact that \(\mathcal{H}\) has infinite VC dimension. Specifically, for any given \(n\) and arbitrarily small \(\varepsilon > 0\), we select a shattered set of size \(N = n/\varepsilon\). 
Next, construct a dataset \(D\) on this shattered set in which all points are labeled \(+1\). Consider a learner that, for any point in \(D\) not included in its training subset, always predicts \(-1\). Since the dataset is shattered, there exists a hypothesis in \(\mathcal{H}\) consistent with any labeling of \(D\), meaning the learner's choice is valid within the hypothesis class. By construction, the learner correctly classifies only the \(n\) training points, while misclassifying the remaining \(N - n\) points. This results in an error rate of \(\frac{N - n}{N} = 1 - \frac{n}{N} = 1 - \varepsilon.\) Taking the limit as \(\varepsilon \to 0\) yields a lower bound of \(1\), establishing the claim:
\[
\mathtt{R}_{\mathcal{H}}^\star(n) \geq 1.
\]
\subsection{Stochastic Convex Optimization}

\subsubsection{Proof of Theorem~\ref{thm:linreg}}

\begin{theorem*}[Theorem~\ref{thm:linreg} Restatement]
Let \(A^\star\) denote the min-norm ERM. Then,
\[
\sup_{D} \frac{L^\star_D(n ; A)}{L^\star_D} =
\begin{cases}
1 & \text{if } n {\geq 2d}, \\
d+1 & \text{if } n = d, \\
\infty & \text{if } n < d,
\end{cases}
\]
where \(D\) ranges over all finite datasets and \( L^\star_D = \inf_{w \in \mathcal{W}} L_D(w) \) denotes the optimal loss. Above, the ratio \(\frac{L^\star_D(n ; A)}{L^\star_D}\) is defined to be \(1\) when both the numerator and denominator are \(0\). If only the denominator is \(0\), the ratio is defined as \(\infty\).

Furthermore, the lower bound in the cased $n<d$ holds for every weakly continuous ERM.
\end{theorem*}
The proof proceeds in two steps: we first derive an upper bound for each of the three cases—\( n \geq 2d \), \( n = d \), and \( n < d \)—and establish show the corresponding lower bounds. 
\begin{remark}\label{rmk:fulldim}  
Without loss of generality, we restrict our analysis to datasets \( D \) whose feature vectors \( \{x \mid (x,y)\in D\} \) span \( \mathbb{R}^d \). Indeed, if the feature vectors lie in a lower-dimensional subspace, we can analyze the problem within that subspace instead. This does not compromise generality, since for any \( w\in\mathbb{R}^d \), we have \( L_D(w) = L_D(w_{||}) \), where \( w_{||} \) is the projection of \( w \) onto the space spanned by the feature vectors (see Lemma~\ref{lem:minnorm}).
\end{remark}
\subsubsection*{Upper Bounds}
\paragraph{Case 1: \( n \geq 2d \)}  
Had the loss functions in linear regression been strictly convex, we could have directly applied Proposition~\ref{prop:car}, yielding a stronger result with a selection budget of \( d+1 \) rather than \( 2d \). However, since the loss functions are not strictly convex and the empirical risk minimizer is not unique, additional care is required to handle this case.  

In linear regression, the loss function for an example \( z = (x, y) \) is given by  
\[
\ell_z(w) = (w \cdot x - y)^2,
\]  
where \( w \in \mathbb{R}^d \). While this function is convex in \( w \), it is not strictly convex for \( d > 1 \) and has multiple empirical risk minimizers. Among them, the min-norm ERM \( A^\star \) selects the solution \( w^\star \) with the smallest Euclidean norm. A key property of this solution is that it always lies within the span of the input vectors \(\{x_i \mid (x_i, y_i) \in D\}\). This is formalized in the following lemma:  

\begin{lemma}\label{lem:minnorm}
Let \( D = \{(x_i, y_i)\}_{i=1}^N \). The min-norm minimizer of $L_D(\cdot)$ is unique and it belongs to  
\(\mathrm{span}\bigl(\{x_i \mid i = 1, \ldots, N\}\bigr).\)
Moreover, any minimizer of $L_D(\cdot)$ can be expressed as \( w_{||} + w_{\perp} \), where  
\(
w_{\perp} \perp \mathrm{span}\bigl(\{x_i \mid i = 1, \ldots, N\}\bigr), 
\)  
and $w_{||}$ is the projection of $w$ on $\mathrm{span}\bigl(\{x_i \mid i = 1, \ldots, N\}\bigr)$.
\end{lemma}  
\begin{proof}
The proof follows from basic linear algebra: any vector \( w \) can be decomposed as \( w = w_{||} + w_{\perp} \), where \( w_{||} \) lies in the span of \( \{x_i \mid i = 1, \ldots, N\} \) and \( w_{\perp} \) is orthogonal to it. Since only the component within the span contributes to the loss \( L_D(w) \), both \( w \) and \( w_{||} \) achieve the same loss. However, \( w_{||} \) has a strictly smaller Euclidean norm, implying that the min-norm loss minimizer must lie within the span. Furthermore, any other loss minimizer can differ from \( w \) only by a component orthogonal to the span, completing the proof.
\end{proof}
By Lemma~\ref{lem:minnorm}, it suffices to find a dataset of \( n \) examples \( D' = \{(x_{i_j}, y_{i_j})\}_{j=1}^n \) such that $w^\star$ minimizes $L_{D'}(\cdot)$ and
\[
w^\star \in \mathrm{span}\bigl(\{x_{i_j} \mid j = 1, \ldots, n\}\bigr).
\]
We establish this using an argument similar to the one in the proof of Proposition~\ref{prop:car}, namely, by applying Carathéodory’s theorem to the gradients. A naive application of Carathéodory’s theorem, as in the proof of Proposition~\ref{prop:car}, guarantees the existence of a weighted subdataset of size \( d+1 \) for which \( w^\star \) minimizes the corresponding weighted loss function. However, the issue is that other minimizers may exist, and some may have smaller norm then \(w^\star\). To ensure uniqueness, we show that selecting \( 2d \) points instead of \( d+1 \) suffices. Moreover, this increase in the selection budget is necessary, as demonstrated in Example~\ref{ex:requirement_n_greater_d}.

By Remark~\ref{rmk:fulldim}, we may assume that \(\mathrm{span}(\{x_i \mid (x_i, y_i) \in D\}) = \mathbb{R}^d\), and hence the minimizer of $L_D(\cdot)$ is unique and equals to $w^\star$. Since the function \( L_D(\cdot) \) is differentiable, we have \( 0 = \nabla L_D(w^\star) \), where \( \nabla L_D(w^\star) \) denotes the gradient of $L_D(\cdot)$ at $w^\star$. By linearity,
\[
0 = \sum_{i=1}^N \frac{1}{N} \nabla \ell_{z_i}(w^\star).
\]
Observe that for \( z = (x, y) \) the gradient of $\ell_z(\cdot)$ is parallel to $x$, specifically: \( \nabla \ell_z(w) = c\cdot x \), where $c$ is the scalar \( c = 2 \operatorname{sign}(wx - y) \sqrt{\lvert \ell_{z}(w) \rvert} \). 
We split the proof into two cases. First, we assume that \( \{\nabla\ell_{z_i}(w^\star)\} \) spans \( \mathbb{R}^d \) and proceed with the proof under this assumption. Later, we will address the complementary case where \( \{\nabla\ell_{z_i}(w^\star)\} \) does not span \( \mathbb{R}^d \) and explain how to handle it.


In the first case, $0$ is an interior point of \(\mathtt{conv}(\{\nabla\ell_{z_i}(w^\star) : 1\leq i \le N\})\), because it is average of all points \(\nabla\ell_{z_i}(w^\star)\) and they span the \(\mathbb{R}^d\). We proceed by employing a variant of Carath\'eodory's theorem, due to Steinitz \citep[p.~8]{matousek}, originally from \cite{Steinitz1916}.
\begin{theorem*}[\cite{Steinitz1916}]
    Consider $X \subset \mathbb{R}^d$ and $x$ a point in the interior of the convex hull of $X$. Then, $x$ belongs to the interior of the convex hull of a set of at most $2d$ points of $X$.
\end{theorem*}
By applying Steinitz's theorem, we can find a subdataset $D'$ for which the corresponding gradients \(\{\nabla \ell_{z_{i_j}}(w^\star)\}\) still contain \(0\) as a \(d\)-interior point of their convex hull. In particular $D'$ must be full dimensional (i.e.\ its feature vectors span $\mathbb{R}^d$) and hence, by Lemma~\ref{lem:minnorm}, \(w^\star\) remains the \underline{unique} minimizer for this subdataset (since the feature vectors of the subdataset span \(\mathbb{R}^d\)). Thus, this subdataset achieves the desired bound.

In the second case, suppose that \( \{\nabla\ell_{z_i}(w^\star)\} \) does not span \( \mathbb{R}^d \), and let \( V \) be the subspace it spans. Consider the subdataset consisting of points whose loss has a non-zero gradient on $w^\star$ \( \{z_{i_j} \mid \ell_{z_{i_j}}(w^\star) \neq 0 \} \). The function \( w^\star \) remains an ERM on this subdataset, but it is not necessarily the min-norm ERM. Let \( w'^\star \) denote the min-norm ERM computed on this subdataset.
Since \( V \) has dimension \( d' < d \), the first part of the proof guarantees the existence of a weighted dataset of \( 2d' \) points whose feature vectors span \( V \) and for which the min-norm ERM outputs \( w'^\star \){, while Lemma~\ref{lem:minnorm} guarantees that \(w^\star\) is also an ERM for this weighted dataset}. 
We augment this set by adding \( d - d' \) points for which \( \ell_{z}(w^\star) = 0 \), ensuring that the full set of \( d + d' \) points spans \( \mathbb{R}^d \). Since adding realizable points does not affect the gradient value at \( w^\star \), it remains an ERM for this augmented dataset. By Lemma~\ref{lem:minnorm}, \( w^\star \) is now the unique minimizer, implying that it is also the minimum-norm solution.

\paragraph{Case 2: \(\boldsymbol{n = d}\).}  
We now show that for any dataset \(D\) in \(\mathbb{R}^d \times \mathbb{R}\) with \(n = d\) {examples}, the ratio \(\sup_{D} \frac{L^\star_D(d ; A)}{L^\star_D}\) is at most \(d+1\). 
The result follows directly from Theorem~5 of \citet{Derezinski17}:
\begin{theorem*}[Theorem 5, \citet{Derezinski17}]  
If the input matrix \(\mathbf{X} \in \mathbb{R}^{d \times n}\) is in general position, then for any label vector \(\mathbf{y} \in \mathbb{R}^{n}\), the expected squared loss (over all \(n\) labeled vectors) of the optimal solution \(\mathbf{w}^*(S)\) for the subproblem \((\mathbf{X}_S, \mathbf{y}_S)\), where the \(d\)-element subset \(S\) is obtained via volume sampling, satisfies:

\[
\mathbb{E}[L(\mathbf{w}^*(S))] = (d + 1) L(\mathbf{w}^*).
\]
If \(\mathbf{X}\) is not in general position, the expected loss is upper-bounded by \((d + 1) L(\mathbf{w}^*)\).
\end{theorem*}
Above, the subproblem \( (X_S, y_S) \) corresponds to applying an ERM on the subdataset \( S \).  
Volume sampling refers to selecting \( S \) from a distribution over all subdatasets of \( D \) of size \( n = d \). Specifically, each subdataset \( S \) is sampled with probability proportional to the volume of the parallelogram spanned by the feature vectors \( \{x : (x, y) \in S\} \).  
As a result, only subdatasets \( S \) whose feature vectors form a basis of \( \mathbb{R}^d \) are sampled. This ensures that every ERM applied to \( S \) returns the unique hypothesis that interpolates the data.  
Overall, this establishes, via a probabilistic argument, that there exists a subdataset \( S \) of size \( n = d \) such that any ERM applied to it returns a hypothesis whose loss is at most \( d+1 \) times the optimal solution \( L_D^\star \).  



\paragraph{Case 3: \(n < d\).}  
For \(n < d\), we only need to prove a lower bound.

\subsubsection*{Lower Bounds}
We now construct examples showing that each of the upper bounds in Theorem~\ref{thm:linreg} is tight.

\paragraph{Case 1: \(n \geq 2d\).}  
    Trivially, for any finite dataset \(D\) and any $n$, we have \(L^\star_D(n; A) \ge L^\star_D\).

\paragraph{Case 2: \(n = d\).}
    Consider the following dataset $D$ consisting of $N=d+1$ examples:
    \[
      D \;=\; \bigl\{(e_i, i \cdot c)\bigr\}_{i=1}^d \;\cup\; \bigl\{(-\sum_{i=1}^d e_i,\, -c \cdot \frac{d(d+1)}{2} -d-1)\bigr\},
    \]
    where \(\{e_i\}_{i=1}^d\) denotes the standard basis of \(\mathbb{R}^d\) and \(c\) is any number greater than \(\sqrt{2}(d+1)\).
A standard calculation shows that the unique optimal solution for the entire dataset is
\[
  w^\star \;=\; (c+1,\, 2c+1,\, \ldots,\, dc+1),
\]
and that the optimal loss is \(L^\star_D = 1\), because each data point contributes a regression loss of \(1\). 

    Next, we show that for any proper subdataset \( S \subset D \), the regression error is at least \( d+1 \). First, observe that if the dataset \(S\) has $d$ distinct examples, then the regression function will incur a loss of \( (d+1)^2 \) on the excluded data point, leading to the desired total loss.
    If \(S\) has less than $d$ distinct examples, we consider two cases:
    
    \begin{enumerate}
    \item[(i)] If \(\bigl\{(-\sum_{i=1}^d e_i,\, -c \cdot \frac{d(d+1)}{2} -d-1)\bigr\} \notin S\), then the regression function defined on \(S\) takes the form 
    \[
    w = \operatorname{mask}_S(c, 2c, \dots, dc),
    \]
    where \(\operatorname{mask}_S\) denotes an operation that sets certain coordinates to zero, specifically those corresponding to any basis vector \( e_i \) for which \( (e_i, ic) \notin S \). This results in an error of at least \( c^2 \) on one of the data points. By the selection of \( c \), \( c^2 > (d+1)^2 \), thereby the desired error bound is acieved.
    \item[(ii)] If \(\bigl\{(-\sum_{i=1}^d e_i,\, -c \cdot \frac{d(d+1)}{2} -d-1)\bigr\} \in S\), then at least two basis vectors, say \( e_i \) and \( e_j \), have their corresponding data points excluded from the dataset. By Lemma~\ref{lem:minnorm}, the min-norm hypothesis $w$ has the same value on both indices, i.e., \( w(i) = w(j) = a \). The sum of the losses on these two excluded data points is given by
    \[
    (iC - a)^2 + (jc - a)^2 \geq \frac{c^2}{2} \geq (d+1)^2.
    \]
    Thus, the total loss remains at least \( d+1 \), ensuring the required error bound.
\end{enumerate}
    We conclude that this dataset \(D\) has \(\frac{L_D^\star(d, A)}{L_D^\star} = d+1\), as stated.



 \paragraph{Case 3: \(n < d\).} 
For each \( \eta > 0 \), we apply the construction from Theorem~\ref{thm:linbinary}, which provides a dataset \( D = \{(x_i, y_i)\} \) satisfying:
\begin{enumerate}
    \item The last coordinate of each \(x_i\) is zero.
    \item Every subset of \(D\) containing at most \(d\) points has feature vectors that are linearly independent.\footnote{In the proof of Theorem~\ref{thm:linbinary}, we argued that any subset consisting of fewer than \(d+1\) points is affinely independent because the dataset \(D\) is drawn at random, and this occurs with probability one. Similarly, with probability one, any set of \(n < d\) points is linearly independent.}
    \item Every homogeneous halfspace in \(\mathcal{H}_d\) has a classification error of at least \(\frac12 - \eta\) on \(D\).
\end{enumerate}

Note that any homogeneous halfspace naturally corresponds to the linear functional defined by its normal vector. Specifically, the halfspace is defined by the equation:
\[
  \langle w, x \rangle \geq 0,
\]
where \(w \in \mathbb{R}^d\) is the normal vector. In this setting, the classification zero-one error on an example \((x_i, y_i)\)  is determined by whether the sign of \(\langle w, x_i \rangle\) matches the label \(y_i\), i.e.,  
\[
\text{sign}(\langle w, x_i \rangle) \neq y_i.
\]
The corresponding regression error on that example is given by \((\langle w, x_i \rangle - y_i)^2\). Notice that each misclassified point—where \(\text{sign}(\langle w, x_i \rangle) \neq y_i\)—incurs a regression error of at least \(1\), since
\[
  (\langle w, x_i \rangle - y_i)^2 \geq 1.
\]
This holds because if \(\text{sign}(\langle w, x_i \rangle) \neq y_i\), then \(\langle w, x_i \rangle\) and \(y_i\) have opposite signs, ensuring that their squared difference is at least \(1\). Consequently, any linear regressor applied to the dataset \(D = \{(x_i, y_i)\}\) incurs a total loss of at least \( \frac{1}{2} - \eta.\)  

Meanwhile, any subset of \(D\) containing fewer than \(d\) points consists of linearly independent feature vectors, and hence there exists a linear regressor that perfectly interpolates that subset (i.e., achieves zero loss).  

To show that the ratio \(\frac{L_D^\star(n;A)}{L_D^\star}\) can be made arbitrarily large, we apply a similar perturbation as in the proof of Theorem~\ref{thm:linbinary}. For each \(\epsilon > 0\), define:
\[
  D_\epsilon 
  \;=\; \bigl\{(x_i + y_i\,\epsilon\,e_d,\; y_i)\bigr\}_{i=1}^N,
\]
where \(e_d\) is the \(d\)-th standard basis vector in \(\mathbb{R}^d\). By continuity, every weakly continuous ERM continues to have non-negligible loss on \(n\)-sized subsets of \(D_\epsilon\) for \(n < d\) and sufficiently small \(\epsilon > 0\). Yet, the entire dataset \(D_\epsilon\) is perfectly realizable by the linear function \(w(x) = \tfrac{1}{\epsilon}x(d)\). Consequently,
\[
  L_{D_\epsilon}^\star \;=\; 0
  \quad \text{while} \quad
  L_{D_\epsilon}^\star(n; A) >\; \frac12 - \eta,
\]
for some small \(\epsilon\), and therefore:
\[
  \sup_D\frac{L_{D}^\star(n; A)}{L_{D}^\star} 
  \;=\; \infty.
\]

\subsubsection{Analysis of Examples for Theorem~\ref{thm:sco}}\label{app:examples}

\vspace{-3pt}\begin{example}[Strict Convexity Requirement]  
\label{ex:strict_convexity}
We provide a \(2\)-dimensional construction, which can be extended to higher dimensions.  

For a vector \(v \in \mathbb{R}^d\), define \(f_v : \mathbb{R}^d \to \mathbb{R}\) by  
\[
f_v(x) =  
\begin{cases}  
0, & \text{if } |\langle v, x \rangle| \leq 1, \\  
|\langle v, x \rangle| - 1, & \text{otherwise},  
\end{cases}  
\]  
where \(\langle v, x \rangle\) is the scalar product of \(v\) and \(x\).  
Note that \(f_v\) is convex but not strictly convex.  

Let \(v_1, v_2, v_3 \in \mathbb{R}^2\) be such that the triangle \(v_1 v_2 v_3\) is regular, and consider the dataset \(D = \{f_{v_1}, f_{v_2}, f_{v_3}\}\) and the function \(f = \frac{f_{v_1} + f_{v_2} + f_{v_3}}{3}\).  
The function \(f\) achieves its global minimum with value \(0\), and the set of global minimizers forms a regular hexagon.  
Moreover, any convex combination of just two functions in \(D\) results in a function whose set of global minimizers is a parallelogram, which strictly contains the hexagon.  
Thus, an ERM \(A^\star\) that outputs a point in the parallelogram but outside the hexagon witnesses an unbounded ratio.  
\begin{figure}[t]
\centering
\begin{tikzpicture}[scale=2]
    \coordinate (A) at (0, 1);
    \coordinate (B) at ({-sqrt(3)/2}, -0.5);
    \coordinate (C) at ({sqrt(3)/2}, -0.5);

    \foreach \i in {0, 1, 2, 3, 4, 5} {
        \coordinate (P\i) at ({cos(60 * \i) * 2/sqrt(3)}, {sin(60 * \i) * 2/sqrt(3)});
    }

    \coordinate (PB) at ({-sqrt(3)}, 1);
    \coordinate (PC) at ({sqrt(3)}, -1);
    
    \fill[blue!70!black, opacity=0.5] (P1) -- (PB) -- (P4) -- (PC) -- cycle;
    \draw[thick, blue!70!black] (P1) -- (PB) -- (P4) -- (PC) -- cycle;
    
    \draw[thick, red] (P0) -- (P1) -- (P2) -- (P3) -- (P4) -- (P5) -- cycle;
    \fill[red, opacity=0.7] (P0) -- (P1) -- (P2) -- (P3) -- (P4) -- (P5) -- cycle;

    \draw[->, thick] (0, 0) -- (A) node[midway, left] {$\mathbf{v}_1$};
    \draw[->, thick] (0, 0) -- (B) node[midway, below left] {$\mathbf{v}_2$};
    \draw[->, thick] (0, 0) -- (C) node[midway, below right] {$\mathbf{v}_3$};

    \node[circle, fill=blue, scale=0.5] at (A) {};
    \node[circle, fill=blue, scale=0.5] at (B) {};
    \node[circle, fill=blue, scale=0.5] at (C) {};

    \foreach \i in {0, 1, 2, 3, 4, 5} {
        \node[circle, fill=red, scale=0.3] at (P\i) {};
    }
\end{tikzpicture}
\caption{A 2D illustration for Example~\ref{ex:strict_convexity}. The three vectors
$(\mathbf{v}_1,\mathbf{v}_2,\mathbf{v}_3)$ define convex functions $f_{\mathbf{v}_1}$,
$f_{\mathbf{v}_2}$, $f_{\mathbf{v}_3}$. The red hexagon marks the common zero set
(minimizers) when all three functions are combined. Removing $f_{\mathbf{v}_3}$ enlarges
the zero set to the blue region (a rhombus), showing how the feasible set for an ERM grows
once strict convexity is violated.}
\label{fig:strict_convexity}
\end{figure}
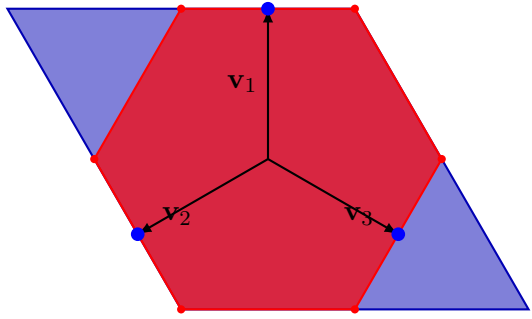
\end{example}

\begin{example}[Necessity of the \(n > d\) Assumption]\label{app:requirement_n_greater_d}
\label{ex:requirement_n_greater_d}
 We construct an example illustrating that when \(n \leq d\), the loss \(L^\star_D(n; A^\star)\) can significantly exceed the loss \(L^\star_D\).


Consider a set of \( d+1 \) points \( \{v_1, \dots, v_{d+1}\} \) in \( \mathbb{R}^d \) arranged as a regular simplex, with each point positioned at a distance of \( 1 \) from the hyperplane spanned by the remaining \( d \) points. For each vertex \( v_i \), define the loss function as 
\[
f_{v_i}(w) = \|w - v_i\|^P,
\]
where \( P > 2 \).

Now, suppose we exclude a point \( v_i \) when forming a subdataset. 
In this case, the minimizer \( w \) must lie within the hyperplane spanned by the remaining \( d \) points: indeed, if \( w \) had a nonzero orthogonal component, say \( w = w_{||} + w_{\perp} \) (where \( w_{||} \) lies in the hyperplane and \( w_{\perp} \) is the orthogonal component), then by the Pythagorean theorem, the loss of \( w_{||} \) would be strictly smaller than that of \( w \).

Since \( w \) lies in the hyperplane, the loss at the omitted point \( v_i \) satisfies \( f_{v_i}(w) \geq 1 \), contributing at least \( \frac{1}{d+1} \) to the loss. Thus, we obtain the lower bound:
\[
  L^\star_D(n; A^\star) \geq \frac{1}{d+1}.
\]
When all \( d+1 \) functions are included, the global minimizer is the centroid of the simplex, leading to:
\[
  L^\star_D = \biggl(\frac{d}{d+1}\biggr)^P.
\]
To establish that \( L^\star_D(n; A^\star) > C \cdot L^\star_D \), we require:
\[
  \frac{1}{d+1} > C \cdot \biggl(\frac{d}{d+1}\biggr)^P.
\]
Notice that this inequality indeed holds for a sufficiently large \( P \). Therefore, when \( n \leq d \), the ratio \( \frac{L^\star_D(n; A^\star)}{L^\star_D} \) can be made arbitrarily large, establishing the necessity of the \( n > d \) condition for the guarantees in Theorem~\ref{thm:sco}.
\end{example}


\subsection{High-Dimensional Mean Estimation}\label{sec:p20p21}

\begin{proposition}\label{prop:subsampling-bounds}
    Consider the stochastic convex optimization problem of estimating the mean of a distribution over \(\mathbb{R}^d\) with squared loss \(\ell_z(h) = \|z - h\|^2\). 
    For every \( n \geq 1 \) and \( d \geq 1 \), let~\( L_D^\star \) denote the optimal loss over the dataset \( D \subseteq \mathbb{R}^d \), and \( L_D^\star(n) \) the optimal loss achieved when selecting only~\(n\) datapoints from \( D \).
    Then, the following bounds hold:
    \[
        \frac{2n}{2n - 1} \leq \sup_{D \subseteq \mathbb{R}^d} \frac{L_D^\star(n)}{L_D^\star} \leq \frac{n + 1}{n}.
    \]
\end{proposition}

\begin{proof}
    We begin with the lower bound:
    \[
        \frac{2n}{2n - 1} \leq \sup_{D \subseteq \mathbb{R}^d} \frac{L_D^\star(n)}{L_D^\star}.
    \]
    To establish this, it suffices to provide, for every \(n\) and \(d\), a distribution \(D\) supported on \(\mathbb{R}^d\) such that the ratio \(\frac{L_D^\star(n)}{L_D^\star}\) is at least \(\frac{2n}{2n - 1}\). The one-dimensional construction used in Theorem~\ref{t:mean} — namely, the dataset with \(2n - 1\) copies of \(0\) and a single copy of \(1\) — achieves this lower bound. Since this dataset can be embedded in \(\mathbb{R}^d\) for any \(d\) (by padding with zeros), the same ratio is preserved. Hence, the lower bound holds in all dimensions.

    We now turn to the upper bound:
    \[
        \sup_{D \subseteq \mathbb{R}^d} \frac{L_D^\star(n)}{L_D^\star} \leq \frac{n + 1}{n}.
    \]
    Consider a dataset \(D \subseteq \mathbb{R}^d\) and let \(\mathcal{D}\) be the uniform distribution over \(D\). Then the optimal \(n\)-point subset average minimizes the loss:
    \[
        L_D^\star(n) = \min_{z_1,\ldots,z_n \in D} L_D\left( \frac{1}{n} \sum_{i=1}^n z_i \right).
    \]
    In particular, since the minimum is always less than or equal to the expectation, we have:
    \[
        L_D^\star(n) \leq \mathbb{E}_{z_1, \ldots, z_n \sim \mathcal{D}} L_D\left( \frac{1}{n} \sum_{i=1}^n z_i \right).
    \]
    By Lemma~\ref{lem:13}, \(L_D(h) = \| h - \mu_D \|^2 + L_D^\star\), and therefore
    \[
        \mathbb{E}_{z_1, \ldots, z_n} L_D\left( \frac{1}{n} \sum_{i=1}^n z_i \right) = L_D^\star + \mathbb{E} \left\| \frac{1}{n} \sum_{i=1}^n z_i - \mu_D \right\|^2.
    \]
    The second term is the variance of the average of \(n\) i.i.d.\ draws from \(D\), which is equal to \(\frac{1}{n} \cdot \mathrm{Var}(D)= \frac{1}{n}\cdot L_D^\star\) (see Proposition~\ref{prop:car}). Therefore,
    \[
        L_D^\star(n) \leq L_D^\star + \frac{1}{n} \cdot L_D^\star.
    \]
    Dividing both sides by \(L_D^\star\) completes the proof of the upper bound.
\end{proof}

\begin{proposition}\label{prop:asymptotic-tightness}
    Under the same setting as Proposition~\ref{prop:subsampling-bounds}, the upper bound becomes tight as the dimension \( d \to \infty \). More precisely,
    \[
        \lim_{d \to \infty} \sup_{D \subseteq \mathbb{R}^d} \frac{L_D^\star(n)}{L_D^\star} = \frac{n + 1}{n}.
    \]
\end{proposition}
   
\begin{proof}
    Fix \( n \in \mathbb{N} \), and for each \( d > n \), consider the dataset \( D_d = \{e_1, e_2, \dots, e_d\} \subset \mathbb{R}^d \), where~\(e_1,\ldots,e_d\) are the standard basis vectors.
    Let \( \mu_d = \frac{1}{d} \sum_{i=1}^d e_i = \left( \frac{1}{d}, \dots, \frac{1}{d} \right) \). Then, the optimal loss is given by:
    \[
        L_{D_d}^\star = \frac{1}{d} \sum_{i=1}^d \left\| e_i - \mu_d \right\|^2 = \frac{d-1}{d}.
    \]
    Now consider \( L_{D_d}^\star(n) \), the minimal loss over all averages of subsets of \( n \) points from \( D_d \). Any such average lies in the convex hull of at most \( n \) of the basis vectors, and hence has at most \( n \) non-zero coordinates. Let \( h \in \text{conv}\{e_{i_1}, \dots, e_{i_n}\} \) be such a minimizer. Then by Lemma~\ref{lem:13}:
    \[
        L_{D_d}(h) = L_{D_d}^\star + \| h - \mu_d \|^2.
    \]
    The squared distance to the mean \( \mu_d = \left(\frac{1}{d}, \dots, \frac{1}{d}\right) \) is minimized when \( h \) satisfies
    \[
        h = \frac{1}{n} \sum_{j=1}^n e_{i_j},
    \]
    which has entries \( \frac{1}{n} \) in \( n \) coordinates and \( 0 \) elsewhere. Therefore,
    \begin{align*}
        \| h - \mu_d \|^2 &= n \left( \frac{1}{n} - \frac{1}{d} \right)^2 + (d - n) \cdot \left( \frac{1}{d} \right)^2 \\
        & = \frac{1}{n} - \frac{1}{d}.
    \end{align*}
    Therefore, the total loss incurred by such a selection is:
    \[
        L_{D_d}^\star(n) = L_{D_d}^\star + \frac{1}{n} - \frac{1}{d}.
    \]
    Dividing by \( L_{D_d}^\star = \frac{d - 1}{d} \), we obtain:
    \[
        \frac{L_{D_d}^\star(n)}{L_{D_d}^\star} = 1 + \frac{1}{L_{D_d}^\star} \left(\frac{1}{n} - \frac{1}{d} \right).
    \]
    As \( d \to \infty \), we have \( L_{D_d}^\star \to 1 \), and hence:
    \[
        \lim_{d \to \infty} \frac{L_{D_d}^\star(n)}{L_{D_d}^\star} = 1 + \frac{1}{n} = \frac{n + 1}{n}.
    \]
    Since this construction achieves the upper bound asymptotically, we conclude:
    \[
        \lim_{d \to \infty} \sup_{D \subseteq \mathbb{R}^d} \frac{L_D^\star(n)}{L_D^\star} = \frac{n + 1}{n}.
    \]
\end{proof}


    

\acks{
Shay Moran and Alexander Shlimovich are supported by the European Union (ERC, GENERALIZATION, 101039692). Views and opinions expressed are however those of the author(s) only and do not necessarily reflect those of the European Union or the European Research Council Executive Agency. Neither the European Union nor the granting authority can be held responsible for them.

Shay Moran is also supported by Robert J.\ Shillman Fellowship, by ISF grant 1225/20, by BSF grant 2018385, and by Israel PBC-VATAT, and by the Technion Center for Machine Learning and Intelligent Systems (MLIS).

Amir Yehudayoff's research is supported by the BSF, 
by the Danish National Research Foundation, 
and the Pioneer Centre for AI, DNRF grant number P1.
}

\bibliography{learning}

\end{document}